\documentclass{article}

\usepackage[accepted]{icml2026}
\usepackage{threeparttable}
\usepackage{arydshln}
\usepackage{booktabs}
\usepackage{times}
\usepackage{epsfig}
\usepackage{graphicx}
\usepackage{amsmath}
\usepackage{amssymb}
\usepackage[english]{babel}
\usepackage{pdfrender}

\usepackage{hyperref}
\usepackage{url}
\usepackage{mathtools}
\usepackage{dpr}
\usepackage{fec}
\usepackage{enumerate}
\usepackage{subcaption}
\usepackage{makecell}
\usepackage{comment}
\usepackage{caption}
\usepackage{subcaption}
\usepackage{enumitem}
\usepackage{amsthm}
\usepackage{bm}
\usepackage{siunitx,tabularx,ragged2e} 
\usepackage{float}
\usepackage{pifont}
\usepackage{soul}
\usepackage{epstopdf}
\usepackage{blindtext}
\usepackage{fancyvrb}
\usepackage{multirow, booktabs}
\usepackage{cleveref}
\usepackage{hhline}
\usepackage{textcomp}
\usepackage{tcolorbox}
\usepackage{MnSymbol}
\usepackage{amssymb,fge}
\usepackage{setspace}
\usepackage{subcaption}
\usepackage{bbding}
\usepackage{cleveref}
\usepackage{hhline}
\usepackage{booktabs}
\usepackage{tcolorbox}
\usepackage[table,xcdraw,dvipsnames]{xcolor}
\usepackage{amssymb} 
\usepackage{makecell}
\usepackage{wrapfig}

\definecolor{StableBlueDark}{RGB}{120, 185, 220}
\definecolor{darkpastelgreen}{rgb}{0.01, 0.75, 0.24}
\definecolor{electriccrimson}{rgb}{1.0, 0.0, 0.25}
\definecolor{navyblue}{rgb}{0.0, 0.0, 0.75}
\definecolor{FrontiersBlue}{RGB}{55, 110, 180}
\definecolor{MSRed}{HTML}{F25022}
\definecolor{MSGreen}{HTML}{7FBA00}
\definecolor{MSBlue}{HTML}{00A4EF}
\definecolor{MSYellow}{HTML}{FFB900}

\definecolor{StableBlueStrong}{RGB}{30, 110, 180}

\hypersetup{colorlinks=true,linkcolor=StableBlueStrong,citecolor=FrontiersBlue,urlcolor=StableBlueStrong}

\newcommand{\tc}[1]{{\color{blue}#1}}

\newcommand{\Var}{\mathrm{Var}}

\newcommand{\gdsq}{g_{\mathrm{DSQ}}}

\newcommand{\sech}{\mathrm{sech}}

\theoremstyle{plain}
\newtheorem{theorem}{Theorem}[section]

\newtheorem{lemma}[theorem]{Lemma}

\theoremstyle{definition}

\newcommand{\algacro}{StableQAT} 
\newcommand{\algacrofourier}{RDFS}

\newcommand{\cmark}{\textcolor{green!70!black}{\small\ding{51}}}
\newcommand{\xmark}{\textcolor{red}{\small\ding{55}}}

\newcommand{\paperabstract}{Quantization-aware training (QAT) is essential for deploying large models under strict memory and latency constraints, yet achieving stable and robust optimization at ultra-low bitwidths remains challenging. Common approaches based on the straight-through estimator (STE) or soft quantizers often suffer from gradient mismatch, instability, or high computational overhead. As such, we propose \algacro{}, a unified and efficient QAT framework that stabilizes training in ultra low-bit settings via a novel, lightweight, and theoretically grounded surrogate for backpropagation derived from a discrete Fourier analysis of the rounding operator. \algacro{} strictly generalizes STE as the latter arises as a special case of our more expressive surrogate family, yielding smooth, bounded, and inexpensive gradients that improve QAT training performance and stability across various hyperparameter choices. In experiments, \algacro{} exhibits stable and efficient QAT at 2-4 bit regimes, demonstrating improved training stability, robustness, and superior performance with negligible training overhead against standard QAT techniques.}

\newif\iffrontiersstyle
\frontiersstyletrue

\newcommand{\microsoftlogo}{%
\raisebox{-0.2ex}{\begin{tikzpicture}[x=0.95ex,y=0.95ex]
\fill[MSRed] (0,1.1) rectangle (1,2.1);
\fill[MSGreen] (1.15,1.1) rectangle (2.15,2.1);
\fill[MSBlue] (0,0) rectangle (1,1);
\fill[MSYellow] (1.15,0) rectangle (2.15,1);
\end{tikzpicture}}%
}

\newcommand{\frontiersheader}{%
\noindent\begin{minipage}{\textwidth}
\small
\newcommand{\instlogoheight}{2.2ex}
\newcommand{\instlogogap}{0.9em}
\begin{tabular*}{\textwidth}{@{\extracolsep{\fill}} p{0.84\textwidth} r}
\raggedright\textbf{\microsoftlogo\hspace{0.4em}\raisebox{0.12ex}{Microsoft}}\hspace{\instlogogap}\raisebox{-0.2ex}{\includegraphics[height=\instlogoheight]{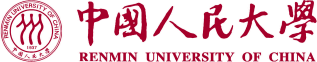}}\hspace{\instlogogap}\raisebox{-0.2ex}{\includegraphics[height=\instlogoheight]{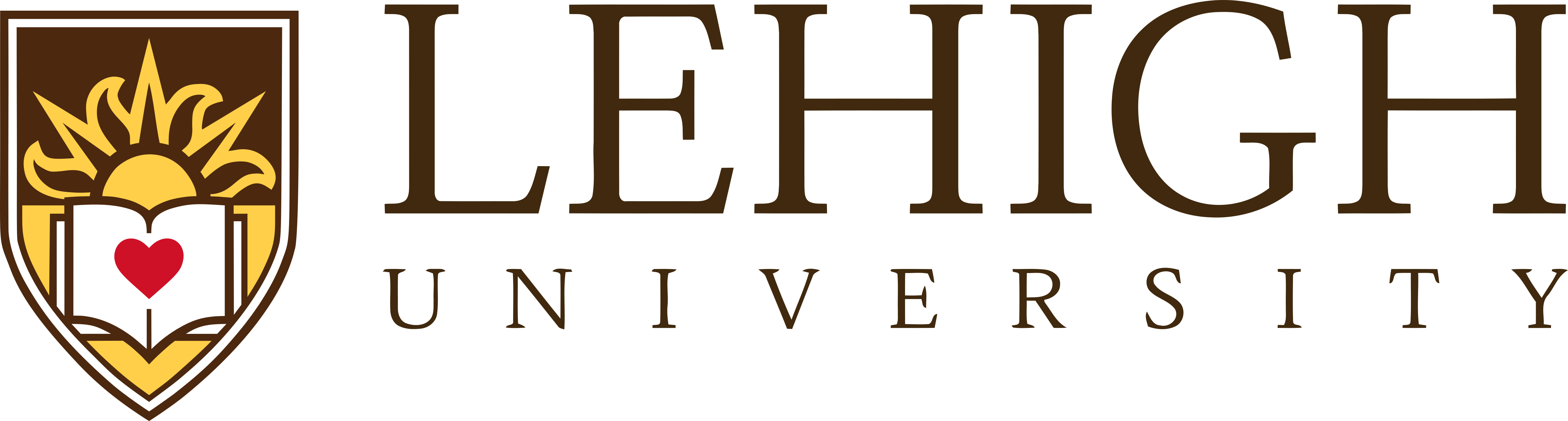}}\hspace{\instlogogap}\raisebox{-0.2ex}{\includegraphics[height=\instlogoheight]{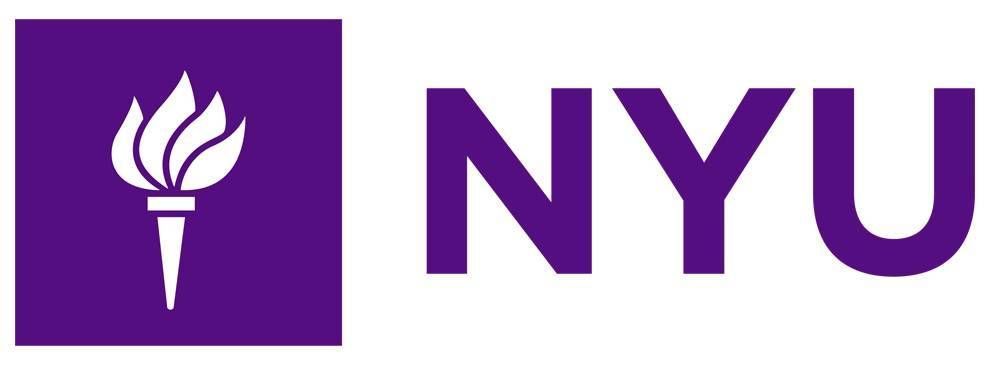}} & 2026-02-17
\end{tabular*}
\end{minipage}\vspace{-0.15em}
}

\newcommand{\frontierslinks}{%
\begin{center}
\small
\raisebox{-0.15ex}{\includegraphics[height=1.7ex]{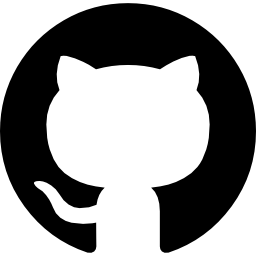}}\hspace{0.35em}\href{https://github.com/microsoft/StableQAT}{\textcolor{navyblue}{\texttt{{https://github.com/microsoft/StableQAT}}}}
\end{center}
}

\newcommand{\frontierssummarybox}{%
\begin{tcolorbox}[
  colback=FrontiersBlue!8,
  colframe=FrontiersBlue!22,
  boxrule=0pt,
  arc=2.8mm,
  left=2.2mm,
  right=2.2mm,
  top=1.2mm,
  bottom=1.2mm
]
\small
\paperabstract
\end{tcolorbox}
}

\icmltitlerunning{\algacro{}: Stable, High-Performing, Efficient QAT}

\makeatletter
\renewcommand{\printAffiliationsAndNotice}[1]{\global\icml@noticeprintedtrue%
  \stepcounter{@affiliationcounter}%
  {\let\thefootnote\relax\footnotetext{\hspace*{-\footnotesep}\ificmlshowauthors #1\fi%
      \forloop{@affilnum}{1}{\value{@affilnum} < \value{@affiliationcounter}}{
        \ifnum\value{@affilnum}>1\hspace{0.9em}\fi
        \textsuperscript{\arabic{@affilnum}}\ifcsname @affilname\the@affilnum\endcsname%
          \csname @affilname\the@affilnum\endcsname%
        \else
          {\bf AUTHORERR: Missing \textbackslash{}icmlaffiliation.}
        \fi
      }\hspace{0.9em}(*) Core contributors.%
      \ifdefined\icmlcorrespondingauthor@text
        \par\smallskip\noindent\hspace*{1.2em}Correspondence to: \icmlcorrespondingauthor@text.
      \else
        {\bf AUTHORERR: Missing \textbackslash{}icmlcorrespondingauthor.}
      \fi
    }
  }
}
\makeatother

\begin{document}

\onecolumn

  \iffrontiersstyle
  \frontiersheader
  \fi

  \icmltitle{
  \resizebox{0.98\textwidth}{!}{\algacro{}: Stable Quantization-Aware Training at Ultra-Low Bitwidths}
  }

  \icmlsetsymbol{core}{*}

  \begin{icmlauthorlist}
    \resizebox{0.975\textwidth}{!}{\mbox{%
      \icmlauthor{Tianyi Chen}{core,msft}
      \icmlauthor{Sihan Chen}{core,renmin}
      \icmlauthor{Xiaoyi Qu}{core,lehigh}
      \icmlauthor{Dan Zhao}{core,nyu}
      \icmlauthor{Ruomei Yan}{msft}
      \icmlauthor{Jongwoo Ko}{msft}
      \icmlauthor{Luming Liang}{msft}
      \icmlauthor{Pashmina Cameron}{msft}
    }}
    \end{icmlauthorlist}

  \icmlaffiliation{msft}{Microsoft}
  \icmlaffiliation{lehigh}{Lehigh University}
  \icmlaffiliation{nyu}{New York University}
  \icmlaffiliation{renmin}{Renmin University of China}

  \icmlcorrespondingauthor{Tianyi Chen}{\href{mailto:Tianyi.Chen@microsoft.com}{\texttt{Tianyi.Chen@microsoft.com}}}

  \iffrontiersstyle
  \frontierslinks
  \frontierssummarybox
  \fi

  \vskip 0.3in

\printAffiliationsAndNotice{}

\iffrontiersstyle
\else
\begin{abstract}
\paperabstract
\end{abstract}
\fi

\section{Introduction}


Large language models (LLMs) are increasingly being deployed under strict constraints on memory bandwidth, energy consumption, and hardware throughput, making full-precision inference impractical at scale. As a result, model quantization has become a key tool for efficient deployment. Although post-training quantization (PTQ) achieves competitive accuracy at 8-bit precision, its performance degrades sharply below 4 bits due to heterogeneous data distributions~\citep{ding2023efficiency} among other challenges, motivating the use of quantization-aware training (QAT) \cite{frantar2023gptq, xiao2023smoothquant, liu2024llmqat} instead.
QAT addresses this limitation by mimicking the errors experienced by the model during inference in training, such as discretization, computation loss, operator fusion/graph-level transformation loss, allowing the model to learn everything via backpropagation \cite{jacob2018integer}.

\textbf{Quantization-aware training (QAT)} exhibits increased optimization fragility as target bitwidths decrease, with both stability and accuracy becoming difficult to maintain at and below 4-bit precision \cite{du2024bitdistiller, panferov2025quest}. Recent studies attribute this to multiple interacting factors, including outlier-dominated distributions, sensitivity to quantizer scaling and clipping, loss plateaus, lack of trainable parameters in standard activation functions, and accumulating approximation errors across layers \cite{choi2018pact, esser2019lsq, beyondoutliers2025}. In spite of the many attempts to mitigate these from different angles \citep{dsmm2024, efficientqat2025}, low-bit QAT for LLMs often remains unstable and yields suboptimal performance; a key underlying reason is that existing approaches are unable to effectively resolve the fundamental mismatch between the non-differentiable quantization operator in the forward pass and continuous gradient-based optimization in the backward pass.

\begin{figure*}[t]
    \centering
\includegraphics[width=0.95\linewidth]{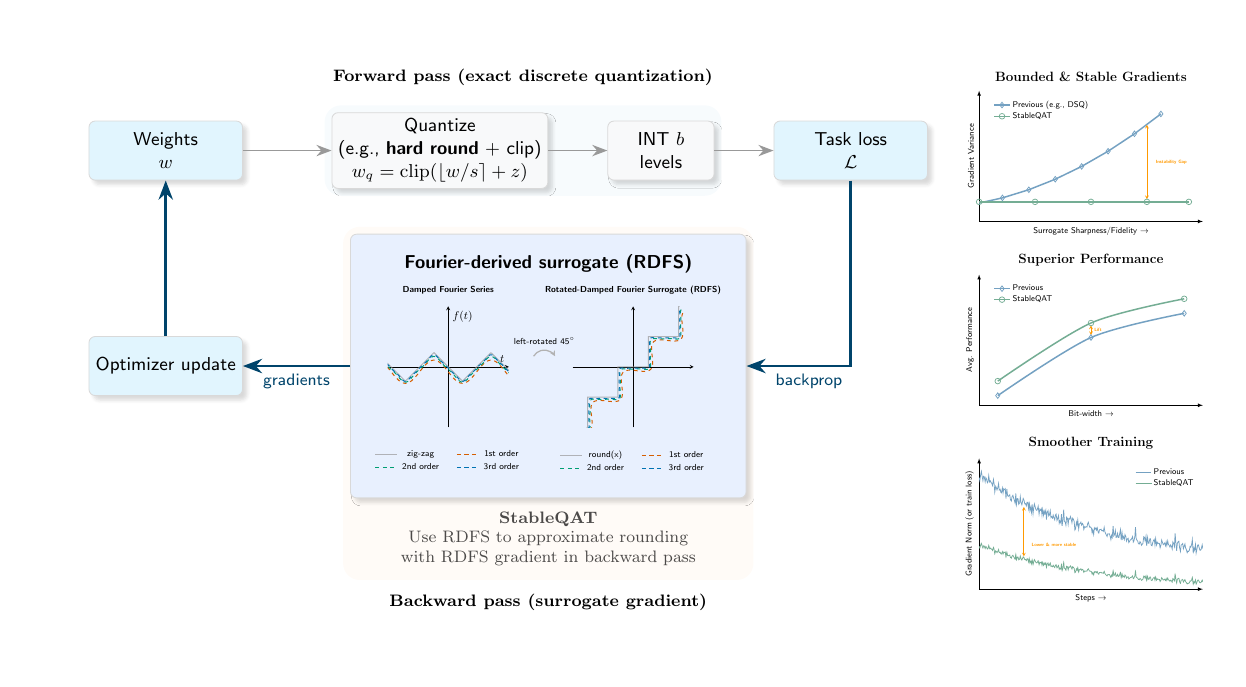}
\vspace{-2em}
    \caption{\textbf{\algacro{}:} a rotated damped Fourier surrogate (RDFS) provides stable and bounded gradients during backpropagation.}
    \label{fig:qatflow}
\end{figure*}

\paragraph{Forward-Backward Mismatch in QAT.} A central challenge in QAT arises from the rounding operator: exact quantization relies on hard rounding to map continuous values to discrete levels but the rounding function is non-differentiable with zero gradient almost everywhere.
To enable backpropagation, the straight-through estimator (STE) is commonly adopted as a surrogate \cite{bengio2013ste, hubara2018quantized}, but STE introduces significant approximation error in the optimization direction, which often leads to unstable training, especially at low bitwidths. Recent approaches address this mismatch by introducing differentiable or stochastic relaxations of quantization, including noise-based training, soft rounding, and smooth approximations that gradually anneal to discrete behavior \cite{gong2019dsq, zhong2023mat, diffq2022, sdq2022,semenov2025smooth}. In attempting to smoothen the optimization process, these methods typically incur additional computational overhead, introduce extra variance, or require careful scheduling to balance fidelity to discrete inference against optimization stability.

\begin{wraptable}{r}{0.52\textwidth}
\vspace{-1.0em}
\centering
\scriptsize
\caption{Qualitative comparison of QAT methods.
}
\begin{tabular}{lccc}
\toprule
\textbf{Characteristic} & \textbf{StableQAT} & \textbf{DSQ} & \textbf{STE} \\
\midrule
Theoretically grounded surrogate & \cellcolor{yellow!20}{\cmark} & \cmark & \xmark \\
Computationally lightweight  & \cellcolor{yellow!20}{\cmark} & \xmark & \cmark \\
Low-bit training \& performance stability & \cellcolor{yellow!20}{\cmark} & \xmark & \xmark \\
\bottomrule
\end{tabular}
\vspace{-1.2em}
\end{wraptable}
\paragraph{Contributions.}
We introduce \textbf{\algacro{}}, a simple but flexible and effective framework to address the optimization bottlenecks of QAT that can be seamlessly integrated into existing training pipelines.
\algacro{} introduces a novel surrogate for the quantization backpropagation: the \textbf{rotated damped Fourier surrogate}(RDFS) models the rounding operator through its spectral structure, yielding a smooth and bounded optimization direction. As a result, \algacro{} provides a theoretically grounded and efficient plug-and-play solution to stabilize and boost QAT performance. Our main contributions are summarized as follows.




\begin{itemize}[leftmargin=6pt]
    \item \textbf{Rotated Damped Fourier Surrogate (RDFS).}
    We propose a novel surrogate for the quantization operator by modeling quantization through Fourier analysis, applying a geometric rotation, and damping the amplitude, yielding a computationally efficient, smooth, bounded, and stable optimization behavior for QAT with negligible implementation complexity.

    \item \textbf{Theoretical Analysis.}
    We provide comprehensive theoretical insights into the properties of our proposed surrogate, including its approximation error and variance, revealing its optimality and clear advantages over STE and prior soft rounding alternatives in terms of stability, conditioning and computational efficiency.

    \item \textbf{Robust Performance \& Stability.}
    Extensive experiments across Large Language Models (LLMs) and Vision Transformers (ViTs) demonstrate that \algacro{} achieves consistently improved training stability and performance at 2--4 bit precision across model scales, without incurring additional computational overhead compared to standard QAT.
\end{itemize}

\vspace{-0.5em}

\section{Preliminary and Related Work}\label{sec.preliminary}

This section reviews the necessary preliminaries and the most closely related works. Additional discussions on post-training quantization (PTQ) and other quantization-aware training (QAT) paradigms are deferred to Appendix~\ref{appendix:more_related_works}.

\subsection{Quantization}
\label{subsec:quantization}

Quantization maps high-precision floating values (e.g., FP32 or FP16) into low-precision discrete representations (e.g., INT8 or INT4), thereby reducing memory footprint and inference-time computational cost. Given a floating-point variable ${x}_\text{original}$ (e.g., weights or activations) and a target bit-width $b$, a quantizer maps ${x}_\text{original}$ to an integer domain $[q_{\min}, q_{\max}]$. For signed quantization, the range is typically $[-2^{b-1}, 2^{b-1}-1]$, and becomes $[0, 2^b - 1]$ for unsigned quantization. 

The quantization process is parameterized by a scaling factor $s \in \mathbb{R}^+$ and a zero-point $z \in \mathbb{Z}$, and is defined as:
\begin{equation}\label{eq:quant_forward}
\begin{split}
{x}_q &= \mathrm{clip}\left(\mathrm{round}({x}_\text{original}/s) + z, q_{\min}, q_{\max}\right) \\
&= \mathrm{clip}\left(\mathrm{round}({x}) + z, q_{\min}, q_{\max}\right),
\end{split}
\end{equation}
where $\mathrm{round}(\cdot)$ denotes the round-to-nearest-integer operator, and $\mathrm{clip}(v, a, b) = \min(\max(v, a), b)$ enforces the feasible range constraint. We denote ${x} \triangleq {x}_\text{original}/s$ for simplicity of notations. 
Here, the scaling factor $s$ (also referred to as the quantization step size) determines the quantization resolution, while the zero-point $z$ controls the alignment between the floating-point and integer domains. Without loss of generality, we assume $z = 0$ throughout the remainder of this paper for notational simplicity. We note that our analysis remains valid for for arbitrary $z$.

\subsection{Straight-Through Estimator (STE)}
\label{subsec:ste}

Training quantized networks poses a fundamental optimization challenge due to the non-differentiability of the rounding operation in Eq.~(\ref{eq:quant_forward}). In particular, the rounding function has zero derivative almost everywhere and is undefined at integer transition points. As a result, its first-order derivative carries no informative signal, causing standard backpropagation to fail with ineffective parameter updates under gradient-based optimization.

To address this, the Straight-Through-Estimator (STE) \cite{bengio2013estimating} has been widely adopted to provide a surrogate to the derivative. STE simply ignores the rounding operation during the backward pass and approximates the differential of the quantized value ${x}_q$ with respect to the input ${x}$ as an identity function. More formally, let $\mathcal{L}$ denote the task loss. The derivative of the loss with respect to the pre-quantized input ${x}$ is then approximated by STE as:
\begin{equation}
\frac{\partial \mathcal{L}}{\partial {x}}
= \frac{\partial \mathcal{L}}{\partial {x}_q} \cdot
{\fcolorbox{black}{white}{$\displaystyle
\frac{\partial {x}_q}{\partial {x}}
$}}_{\hspace{0.1em}\text{STE}}
\approx
\frac{\partial \mathcal{L}}{\partial {x}_q},
\label{eq:ste}
\end{equation}
where the true derivative $\partial {x}_q / \partial {x}$ is replaced by an identity surrogate. Although STE makes QAT practically feasible, it introduces substantial gradient mismatch due to the large deviation between the straight-through surrogate and the rounding operator. This mismatch injects significant optimization noise, leading to biased and unstable gradient updates, impaired convergence behavior, and ultimately sub-optimal performance, especially at ultra-low bitwidths.

\subsection{Soft Surrogate Quantization}

Soft surrogate quantization methods more directly target the rounding operation by introducing continuous relaxations. DSQ~\cite{gong2019dsq} is a representative approach that employs a parameterized \texttt{tanh} function whose sharpness is gradually increased during training. More recent work~\citep{semenov2025smooth} utilizes a \texttt{sigmoid} function to smooth the rounding operator in~\autoref{eq:quant_forward}.

While these approaches capture richer structural information of the quantization process, they still suffer from the inherently ill-conditioned optimization landscape induced by the rounding operator. Moreover, they rely on computationally expensive \texttt{exp}-based functions, significantly slowing down training, typically by \textbf{3$\times$--5$\times$} in both forward and backward passes (see \autoref{fig:train_loss_comparison}). Other soft quantization methods introduce additional relaxation regularizers; however, they continue to face similar optimization challenges stemming from the underlying rounding behavior.



\section{\algacro{}}\label{sec:method}

\textbf{\algacro{}} is a simple QAT framework that directly addresses the critical forward-backward mismatch of low-bit optimization, achieving improved training stability and performance without incurring additional computational cost (\autoref{fig:qatflow}). At its core, \algacro{} is built upon a \textbf{Rotated Damped Fourier Surrogate (\algacrofourier{}) of the quantization operator}. The key idea is to model the discrete quantization process via a Fourier series and derive a smooth, analytically grounded surrogate for backpropagation through a geometric rotation of this representation. This construction yields a stable and well-behaved optimization direction for gradient-based training.

\paragraph{Backward Pass via \algacrofourier{}:}
As illustrated in~\autoref{fig:fft_zigzag}, we approximate the rounding operator using a Rotated Damped Fourier Surrogate (\algacrofourier{}). 
By preserving the richer structural information of the rounding operation, \algacrofourier{} provides more informative signal and addresses the ill-defined Jacobian $\partial {x}_q/\partial {x}$, thereby stabilizing and improving gradient-based QAT.
Formally, the derivative with respect to the input ${x}$ is computed by \algacrofourier{} as:
\begin{equation}
\frac{\partial \mathcal{L}}{\partial {x}}
= \frac{\partial \mathcal{L}}{\partial {x}_q} \cdot
{\fcolorbox{StableBlueStrong}{white}{$\displaystyle \frac{\partial {x}_q}{\partial {x}}$}}_{\hspace{0.1em}\text{\textcolor{StableBlueStrong}{\algacrofourier}}}
\approx
\frac{\partial \mathcal{L}}{\partial {x}_q}\cdot g({x}, {x}_q),
\label{eq:rofs}
\end{equation}
where $g({x}, {x}_q)$ is the RDFS surrogate and explicitly defined as:
\begin{equation}\label{eq:rotated_fourier_surrogate}
g({x}, {x}_q) =
\frac{1 - A\sqrt{2}\pi \sum_{m=0}^{M} \frac{(-1)^m}{2m+1} \cos\!\big((2m+1)\pi({x}+{x}_q)\big)}
{1 + A\sqrt{2}\pi \sum_{m=0}^{M} \frac{(-1)^m}{2m+1} \cos\!\big((2m+1)\pi({x}+{x}_q)\big)}.
\end{equation}
In practice, we truncate the series and typically use the first-order \algacrofourier{} (i.e., $M=0$):
\begin{equation}
\label{eq:rdfs_first_order}
\fcolorbox{StableBlueStrong}{white}{$\displaystyle
g(x,x_q)=
\frac{1 - A\sqrt{2}\pi \cos\!\big(\pi(x+x_q)\big)}
     {1 + A\sqrt{2}\pi \cos\!\big(\pi(x+x_q)\big)}
$}.
\end{equation}
Here, $A$ is a tunable hyperparameter controlling the sharpness of the surrogate. The derivation, based on a geometric rotation and Fourier analysis, is provided in the following. 

\textbf{STE as Special Case of \algacrofourier{}.} We note that with $A = 0$, \algacrofourier{} reduces to the identity mapping, which corresponds exactly to STE. In other words, STE is a special case of \algacrofourier{} that \algacrofourier{} strictly generalizes; as such, \algacrofourier{} is able to carry richer structural information during training.

\subsection{Rotated Damped Fourier Surrogate (\algacrofourier)}\label{sec:rdfs}


\begin{figure}
    \centering
    \includegraphics[width=\linewidth]{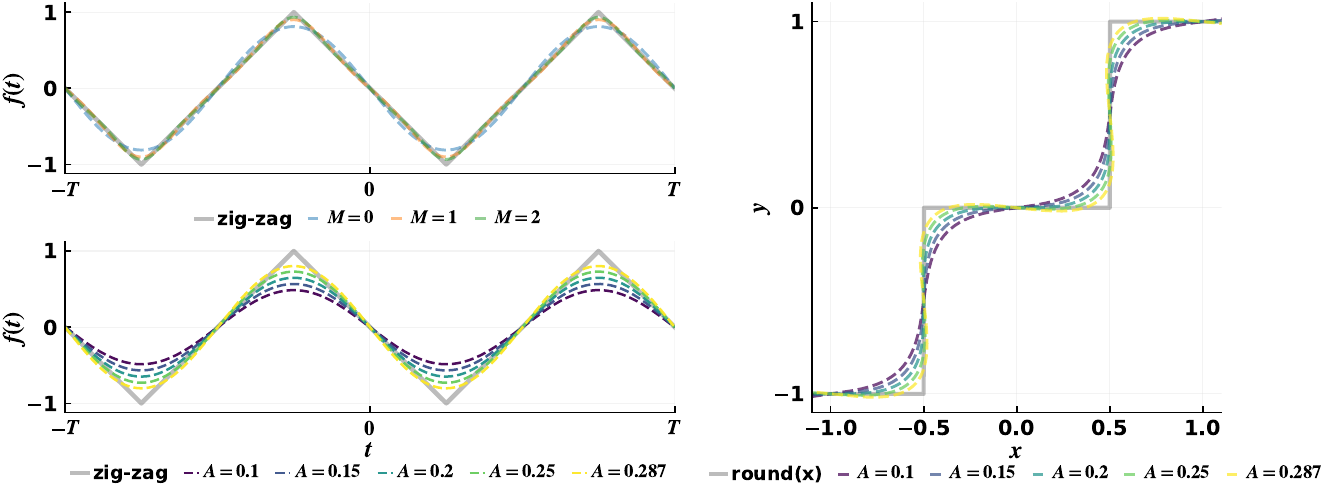}
    \caption{
    Rotated–Damped Fourier Surrogate (\algacrofourier{}) under different orders and amplitudes.
    Higher orders improve approximation fidelity but increase the computational complexity. The first-order truncation ($M=0$) has already captured the essential structure while providing stable and efficient optimization.
    }
    \label{fig:fft_zigzag}
\end{figure}

\paragraph{Coordinate Rotation.} We observe that applying a $45^\circ$ counterclockwise rotation to the coordinate system $(x, x_q)$ transforms the staircase-shaped rounding operator into a periodic, continuous, piecewise-linear function, commonly referred to as a triangle wave function (\autoref{fig:fft_zigzag}).

More formally, we introduce a rotated coordinate system $(t, f(t))$ obtained by rotating the original coordinates $(x, x_q)$ by an angle $\theta = 45^\circ$ counterclockwise. The resulting coordinates are given by
\begin{equation}\label{eq:t_and_x_and_y}
t = \frac{x + x_q}{\sqrt{2}}, \quad f(t) = \frac{- x + x_q}{\sqrt{2}}.
\end{equation}
Under the rotated coordinate system, the rounding operator becomes a periodic zig-zag function with fundamental period $T = \sqrt{2}$. Specifically, the transformed function can be expressed as a centered triangle wave:
\begin{equation}\label{eq:rotated_round_func}
f(t) = \frac{1}{2\sqrt{2}}\left(1 - 4\left|r(t) - \frac{1}{2}\right|\right),
\end{equation}
where $r(t) = \left\{\frac{t - T/4}{T}\right\}$ denotes the phase-shifted fractional part operator, and $\{\cdot\}$ extracts the fractional component by discarding the integer part. The phase shift $T/4$ centers the triangle wave symmetrically around zero, which will be convenient for subsequent Fourier analysis.

\paragraph{Fourier Approximation.}
Since the rotated rounding function $f(t)$ is periodic and square-integrable, it admits a Fourier series expansion. Applying the Fourier transformation to~\autoref{eq:rotated_round_func}, we obtain the following closed-form Fourier approximation for the rotated rounding function $f(t)$:
\begin{equation}
\begin{split}
f(t) &\approx -A\sum_{m=0}^{\infty}
\frac{(-1)^m}{(2m+1)^2}
\sin\!\left((2m+1)\sqrt{2}\pi\, t\right),
\end{split}
\end{equation}
where $(2m+1)\sqrt{2}\pi$ denotes the angular frequency of the $m$-th sinusoid term. 
The scalar $A$ is an amplitude coefficient, as illustrated in \autoref{fig:fft_zigzag}, which controls the sharpness of the resulting curve.
Notably, \textbf{instead of fixing $A$ to the vanilla Fourier amplitude $A=2\sqrt{2}/\pi^2$, we treat it as a tunable parameter.} An appropriate choice of $A$ can yield improved training stability and optimization behavior; see Section \ref{sec:amplitude} for a detailed discussion. 

We further consider a truncated Fourier approximation as 
\begin{equation}
f_M(t) = -A\sum_{m=0}^{M}
\frac{(-1)^m}{(2m+1)^2}
\sin\!\left((2m+1)\sqrt{2}\pi\, t\right),
\end{equation}
where $M$ refers to the order of Fourier approximation. Higher-order terms encode finer structural details of the rounding operator but incur increasing computational cost.

\paragraph{Inverse Rotation.}


To employ the rotated Fourier approximation as a gradient surrogate, we map $f(t)$ back to the original coordinates $(x, x_q)$. Based on the coordinate transformation in \eqref{eq:t_and_x_and_y}, the Fourier expansion of $f(t)$, and the chain rule, we derive the general form of the \textit{rotated damped Fourier surrogate} (RDFS) for the rounding operator, which provides a smooth and well-conditioned approximation to the  non-differentiable and ill-conditioned quantization mapping:
\begin{equation}\label{eq:rdfs_M}
\frac{\partial x_q}{\partial x} 
=
\frac{
1 - A\sqrt{2}\pi
\sum_{m=0}^{M}
\frac{(-1)^m}{2m+1}
\cos\!\left((2m+1)\pi(x + x_q)\right)
}{
1 + A\sqrt{2}\pi
\sum_{m=0}^{M}
\frac{(-1)^m}{2m+1}
\cos\!\left((2m+1)\pi(x + x_q)\right)
}.
\end{equation}
In practice, we retain only the first-order term ($M=0$), which provides an effective trade-off between approximation fidelity and negligible computational overhead. Under this setting, the RDFS in~\autoref{eq:rdfs_M} reduces to its first-order form as \autoref{eq:rdfs_first_order}. A complete derivation of the RDFS is provided in Appendix~\ref{appendix:fourier_rounding}.

\subsection{Damped Amplitude and Gradient Stabilization}\label{sec:amplitude}

The choice of the amplitude parameter $A$ requires careful consideration, as it directly governs the trade-off between approximation fidelity and optimization stability.

\paragraph{Effect of Amplitude.}
The amplitude $A$ controls the sharpness of the rotated Fourier surrogate and how closely it approximates the true rounding operator.
When $A$ is small (approaching STE as $A \to 0$), the surrogate is overly smooth, which weakens its ability to capture the fine-grained structure of rounding and leads to biased or ineffective gradient signals.
Increasing $A$ (while remaining within the admissible range implied by the Fourier construction) improves approximation accuracy by bringing the surrogate closer to hard rounding. However, as the surrogate sharpens, it progressively inherits the pathological optimization behavior of the true rounding operator: gradients vanish over large regions and become highly unstable near discontinuities, resulting in gradients vanishing or exploding; this trade-off motivates the use of a damped amplitude to balance expressiveness and numerical stability.

\paragraph{Well-Conditioned Surrogate Region.}
Our analysis reveals the existence of an \emph{ill-conditioned amplitude regime} in which the rotated Fourier surrogate becomes nearly tangential to horizontal plateaus as $A$ approaches $1/(\sqrt{2}\pi)$, leading to severely attenuated gradients. To avoid this, our design explicitly excludes this region. As shown in the amplitude sensitivity study in~\autoref{fig:rdfs_amplitude}, this ill-conditioned regime manifests as a pronounced performance ``trough'' in our empirical results, (a well in the performance curve), which closely aligns with our theoretical conditioning analysis, demonstrating that proper amplitude bounds are critical for stability.

\paragraph{Practical Amplitude Selection.}
Although Fourier analysis suggests a theoretically admissible amplitude of $A = \tfrac{2\sqrt{2}}{\pi^2}$, strictly adhering to this value is suboptimal in practice for large-scale LLM training due to the aforementioned conditioning issues.
We therefore treat $A$ as a damped and tunable parameter. In all experiments, we select a default value of $A = 0.21$, which lies safely outside the ill-conditioned regime but is still sufficiently expressive to to capture important structural information of the rounding operator to yield effective gradient signal. This choice provides a stable and robust operating point across models and training settings. More sophisticated strategies such as dynamically evolving $A$ during training may further improve performance, but we leave these designs as future work.

\section{Theoretical and Efficiency Analysis}
\label{sec:theorem}
\subsection{Theoretical Comparison with STE}
We compare STE and \algacro{} using surrogate approximations of the rotated rounding function~\eqref{eq:rotated_round_func}. Due to periodicity, we focus on a single period $[0,T]$ without loss of generality. 
As the rotated function is square-integrable, we conduct our analysis in the space $L^2([0,T])$.
\begin{theorem}\label{thm:best_L2_ste}
Let $f \in L^2([0,T])$ and let $f_n$ be the $n$th partial Fourier sum of $f$. Then, for all $n \in \mathbb{N}$,~\\
(i) $f_n$ is the unique minimizer of the $L^2$ approximation error among trigonometric polynomials of degree at most $n$;~\\
(ii) For $n \geq 1$, $\|f - f_n\|_2 < \|f - f_0\|_2$ if and only if $f$ is non-constant almost everywhere.
\end{theorem}
\vspace{-4mm}
\begin{proof}
See Appendix~\ref{appendix:thm5_1}.
\end{proof}
\vspace{-2mm}
\textbf{Remark I.} Theorem~\ref{thm:best_L2_ste}(i) implies that the $n$th partial Fourier sum is the optimal $L^2$ surrogate of the rotated rounding function among trigonometric polynomials of degree at most $n$. For STE, the surrogate is restricted to constant functions. In contrast, \algacro{} admits surrogates from trigonometric polynomials of degree at most $n\in \mathbb{N}$, making STE a special case of \algacro{} with $n=0$.

\textbf{Remark II.} Theorem~\ref{thm:best_L2_ste}(ii) implies that, unless the function is constant almost everywhere, its surrogate drawn from trigonometric polynomials of degree of at least 1 achieves a strictly smaller $L^2$ approximation error than those restricted to constant functions. Consequently, StableQAT admits surrogate functions that are provably closer to the rotated rounding function than STE in the $L^2$ sense.
\subsection{Theoretical Comparison with DSQ}
Next, we compare DSQ and \algacro{} in terms of gradient stability by analyzing the statistical properties of their surrogate gradients
over the clipping range. Our analysis focuses on the asymptotic regime in which the surrogate function closely approximates the rounding function. Such setting is particularly relevant for low-bit QAT (2--4 bits), where capturing discrete quantization effects is critical.
\begin{theorem}\label{thm:asymptotic_dsq}
Let $\gdsq(\cdot;\alpha)$ and $g_{\emph{\algacro{}}}(\cdot;A)$ be the gradient surrogates of DSQ and \algacro{}, parameterized by $\alpha > 0$ and $A\in (0, \frac{1}{\sqrt{2}\pi})$, respectively. Let random variable $\xi$ be uniformly distributed on the interval $[l,u]$, i.e., $\xi \sim U(l,u)$. Consider the asymptotic regimes in which the surrogate functions move close to the rounding function. Then, the limits of the expectations satisfy
\begin{equation*}
\begin{aligned}
\lim_{\alpha \to 0^+}
\E_{\xi \sim U(l,u)}\!\left[\gdsq(\xi;\alpha)\right]
&= 1, \\
\lim_{A \to \left(\frac{1}{\sqrt{2}\pi}\right)^-}
\E_{\xi \sim U(l,u)}\!\left[g_{\emph{\algacro{}}}(\xi;A)\right]
&= \frac{4}{\pi} - 1,
\end{aligned}
\end{equation*}
and the limits of the variances satisfy
\begin{equation*}
\begin{aligned}
\lim_{\alpha \to 0^+}
\Var_{\xi \sim U(l,u)}\!\left[g_{\emph{DSQ}}(\xi;\alpha)\right]
&= \infty, \\
\lim_{A \to \left(\frac{1}{\sqrt{2}\pi}\right)^-}
\Var_{\xi \sim U(l,u)}\!\left[g_{\emph{\algacro{}}}(\xi;A)\right]
&= \frac{16}{3\pi} - \frac{16}{\pi^2}.
\end{aligned}
\end{equation*}
\end{theorem}
\vspace{-4mm}
\begin{proof}
See Appendix~\ref{appendix:thm5_2}.
\end{proof}
\vspace{-2mm}


\textbf{Remark III.} Theorem~\ref{thm:asymptotic_dsq} shows that DSQ's gradient variance diverges as it sharpens towards the rounding function, whereas \algacro{} maintains bounded variance ($\approx0.076$) at maximum sharpness. This distinction directly impacts training stability: bounded variance ensures consistent gradient magnitudes, while divergent variance may cause gradient explosions. Thus, gradient stability is a key advantage of \algacro{} over DSQ as present in \autoref{fig:exp-var-trend}.

\begin{figure}[h]
\centering
\includegraphics[width=\linewidth]{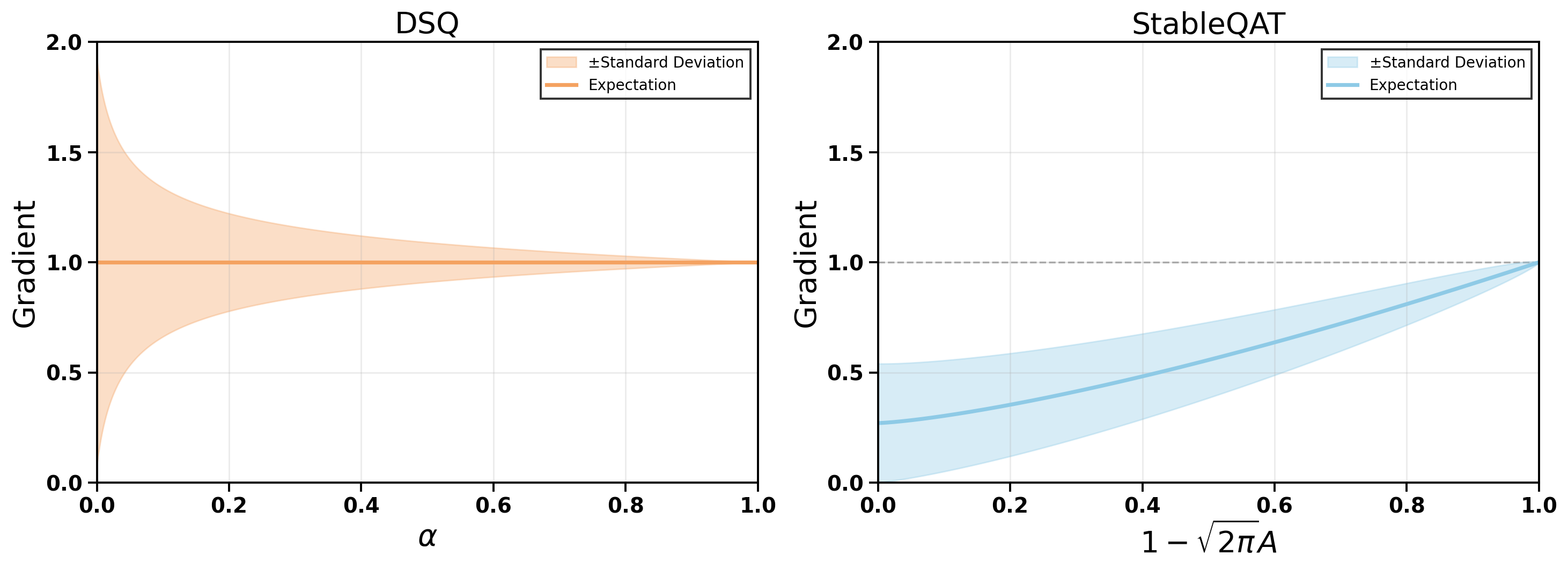}
\caption{\textbf{Gradient spread of DSQ and StableQAT.} As the surrogate sharpens towards the rounding operator, DSQ exhibits exploding variance, while StableQAT shows bounded variance.}
\label{fig:exp-var-trend}
\end{figure}

\subsection{Efficiency Comparison}\label{sec:efficiency_comparison}

We further show that \algacrofourier{} achieves nearly identical computational efficiency to STE despite being faster and more lightweight than DSQ. This makes \algacro{} suitable as a plug-and-play surrogate for a wide range of QAT pipelines, providing consistent performance gains without additional computational/memory cost.

\paragraph{Computational Efficiency: \texttt{cosine} vs.\ \texttt{exp}.} Soft quantization methods such as DSQ~\citep{gong2019dsq} and SigmoidQuant~\citep{semenov2025smooth} approximate rounding using \texttt{sigmoid} or \texttt{tanh}, which can rely on expensive exponential evaluations. On hardware, \texttt{exp} can require high-order polynomial approximation, and additional numerical-stability handling, leading to higher latency and register pressures~\citep{muller2018handbook,nvidia_cuda_guide}. In contrast, trigonometric functions like \texttt{cosine} and \texttt{sine} operate on bounded inputs, admit lower-degree polynomial approximations, and avoid saturation handling. Consequently, \algacrofourier{} exhibits lower and more predictable execution cost, which is especially beneficial in QAT where surrogate gradients are repeatedly evaluated.

\paragraph{\algacrofourier{} is Fusion-Friendly.}
\algacrofourier{} consists only of elementary arithmetic and a single trigonometric operation, without auxiliary states or annealing schedules. Its low register footprint and branch-free structure make it well suited for kernel fusion in modern deep learning runtimes, e.g., CUDA~\citep{nvidia_cuda_guide}, and Triton~\citep{tillet2019triton}. 

\paragraph{Negligible Computational Overhead.} We benchmark \algacrofourier{} on LLaMA-3-1B with batch size 4 and sequence length 128. As shown in \autoref{fig:time_space_complexity_comparison}, \algacro{} matches STE in backward-pass latency and memory usage, while being up to $5\times$ more efficient than \texttt{exp}-based surrogate like DSQ.

\begin{figure}
    \centering
    \includegraphics[width=\linewidth]{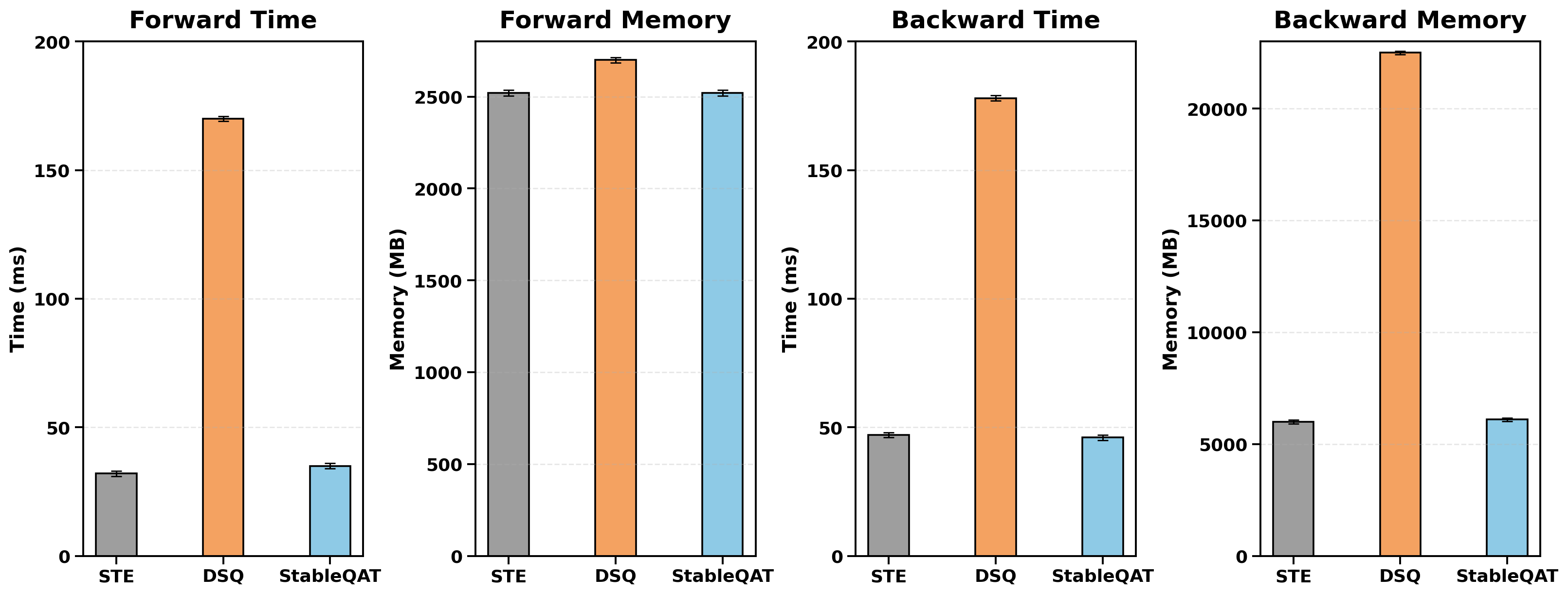}
    \caption{\textbf{Time and space cost comparison (lower is better).} 
    }
    \label{fig:time_space_complexity_comparison}
\end{figure}

\section{Experiments}

We perform comprehensive evaluations of \algacro{} from three complementary perspectives:
\textit{(i)} model performance,
\textit{(ii)} training stability, and
\textit{(iii)} ablative studies to assess the impact of key design choices.
Our experiments span both Large Language Models (LLMs) and Vision Transformers (ViTs) in Appendix~\ref{appendix:vit} to demonstrate the generality, effectiveness, and robustness of \algacro{} across multiple training scenarios.

\subsection{Large Language Models}

\begin{table*}[t]
    \centering
    \caption{LLaMA-3.2-1B results. \algacro{} at 4-bit outperforms 16-bit baseline, while 3-bit remains competitive.}
    \label{tab:llama3-1b}
    \resizebox{\linewidth}{!}{
    \begin{tabular}{cccccccccccccccccc}
    \toprule
   \rowcolor{gray!20} Bits & Method & Setting & Arc-e & Arc-c & Boolq & Hellaswag & Openbookqa & Piqa & SciQ   & Winogrande & Avg & $\Delta$ & GPU Days & SpeedUp\\
    \midrule
    16 & Baseline & Baseline & 60.61 & 36.01 & 63.70 & 63.78 & 36.60 & 74.43 & 88.40 & 60.06 & 60.45 & -- & -- & --\\ 
    8 & Baseline no QAT & Baseline (no QAT)& 60.06 & 36.01 & 64.01 & 63.71 & 36.40 & 74.32 & 88.70 & 59.59 & 60.35 & -- & -- & -- \\ 
    6 & Baseline no QAT & Baseline (no QAT)& 61.15 & 36.35 & 62.60 & 62.79 & 37.40 & 74.27 & 88.30 & 59.91 & 60.35 & -- & -- & --\\
    \midrule
    4 & Baseline (no QAT) & Baseline (no QAT) & 43.31 & 27.05 & 52.57 & 46.75 & 31.20 & 65.83 & 78.00 & 53.12 & 49.73 & -- & -- & -- \\ 
    4 & ParetoQ & 10B Tokens \& lr=1e-5  & 59.89 & 30.97 & 59.27 & 56.07 & 34.40 & 71.65 & 88.80 & 56.83 & \underline{57.24} & -- &  \textbf{0.45} & -- \\
    4 & DSQ & 10B  Tokens \& lr=1e-5  & 55.01 & 31.66 & 61.16 & 55.75 & 34.20 & 71.87 & 86.00 & 56.27 & 56.49 & -- & 0.64 & --\\
    \rowcolor{blue!10}4 & \textbf{\algacro{}} & 10B  Tokens \& lr=1e-5 & 60.48 & 32.17 & 62.02 & 56.54 & 32.60 & 72.63 & 89.60 & 56.91 & \textbf{57.87} & \textbf{\textcolor{green!60!black}{+1.38}} &  \textbf{0.45} & \textbf{\textcolor{green!60!black}{1.42x}}\\
    \midrule
    {4} & ParetoQ & 20B Tokens \& lr=1e-4 & 65.49 & 36.86 & 63.82 & 61.22 & 39.40 & 74.70 & 86.70 & 59.51 & 60.96 & -- & \textbf{0.90} & -- \\
    {4} & DSQ & 20B Tokens \& lr=1e-4 & 65.45 & 37.37 & 63.73 & 62.04 & 39.40 & 74.16 & 86.40 & 59.98 & \underline{61.07} & -- & 1.30 & --\\
    \rowcolor{blue!10}4 & \textbf{\algacro{}} & 20B Tokens \& lr=1e-4 & 65.74 & 37.54 & 64.01 & 61.15 & 40.00 & 74.59 & 86.40 & 60.46 & \textbf{61.24} & \textbf{\textcolor{green!60!black}{+0.25}} & 0.91 & \textbf{\textcolor{green!60!black}{1.43x}}\\
    \midrule
    3 & Baseline (no QAT) & Baseline (no QAT) & 24.74 & 26.30 & 40.64 & 26.11 & 29.00 &  50.49 & 24.90 &  48.30 & 33.81 & -- & -- & -- \\
    3 & ParetoQ & 10B Tokens \& lr=1e-5& 30.85 & 22.10 & 48.53 & 29.19 & 27.80 & 55.33 & 50.30 & 47.28 & \underline{38.92} & -- & {0.46} & -- \\
    3 & DSQ & 10B Tokens \& lr=1e-5  & 29.88 & 23.63 & 46.27 & 29.31 & 28.00 & 53.86 & 46.90 & 48.54 & 38.30 & -- & 0.64 & -- \\
    \rowcolor{blue!10}3& \textbf{\algacro{}} & 10B Tokens \& lr=1e-5 & 38.55 & 23.98 & 59.24 & 33.14 & 27.00 & 59.63 & 67.70 & 52.17 & \textbf{45.18} & \textbf{\textcolor{green!60!black}{+6.88}} & \textbf{0.45} & \textbf{\textcolor{green!60!black}{1.45x}} \\
    \midrule
    3 & ParetoQ & 20B Tokens \& lr=2e-4 & 64.24 &35.69 &  61.23 & 58.91 & 37.40 & 73.32 & 87.50  & 58.27 & \underline{59.57} & -- & \textbf{0.91} & --\\
    3 & DSQ & 20B Tokens \& lr=2e-4  &  60.77 & 32.85 & 59.76 & 56.21 & 36.60 & 72.42 & 83.10  & 57.62 & 57.42 & -- & 1.29 & -- \\
    \rowcolor{blue!10}3 & \textbf{\algacro{}} & 20B Tokens \& lr=2e-4 & 64.06 & 36.43  & 63.24 & 59.49 & 37.80 & 73.67 & 86.7  & 59.59 & \textbf{60.12} & \textbf{\textcolor{green!60!black}{+2.70}} & \textbf{0.91}
    & \textbf{\textcolor{green!60!black}{1.42x}}
    \\
    \midrule
    2 & ParetoQ & 30B Tokens \& lr=1e-4 & 60.02 & 34.22 & 57.65 & 55.63 & 35.80 & 72.91 & 82.90 & 59.19 & \underline{57.29} & -- & 1.38 & --  \\
    {2} & DSQ & 30B Tokens\& lr=1e-4 & 59.13 & 32.59 & 58.72 & 55.65 & 36.20 & 72.03 & 80.50 & 56.27 & 56.39 & -- & 1.95 & --\\
    \rowcolor{blue!10} 2 & \textbf{\algacro{}} & 30B Tokens\& lr=1e-4  & 61.53 & 32.94 & 63.00 & 56.24 & 37.40 & 72.63 & 83.10 & 57.77 &  \textbf{58.08} & \textbf{\textcolor{green!60!black}{+1.69}} & \textbf{1.36} & \textbf{\textcolor{green!60!black}{1.43x}}\\
    \bottomrule
  \end{tabular}
  }
\end{table*}

\paragraph{Experimental Setup.}
As \algacro{} is designed to resolve the forward-backward mismatch of the rounding operator in QAT, 
we focus our comparison on methods that differ primarily in their gradient surrogates. In particular, we compare against ParetoQ, a recent state-of-the-art STE-based QAT method that represents perhaps the strongest instantiation for low-bit QAT, and DSQ, a canonical soft-rounding surrogate that smooths the quantization operator. Together, these two baselines cover the dominant surrogate paradigms used in modern QAT, serving as a representative set of baselines to compare with \algacro{}.

We follow the experimental setup of ParetoQ and evaluate  LLaMA-3-1B and LLaMA-3-3B~\citep{grattafiori2024llama3herdmodels}, under weight 2--4 bit quantization. Experiments are conducted via 16 H100 GPUs with a pretokenized training corpus. Performance is assessed on a popular suite of benchmarks, including ARCEasy, ARC-Challenge, BoolQ, HellaSwag, OpenBookQA, PIQA, SciQ, Winogrande~\citep{lmevalharness}. 
Both ParetoQ and DSQ are reproduced from their official repositories with recommended hyperparameters and training schedules.
The training corpus is constructed by mixing SlimPajama~\citep{cerebras2023slimpajama} and FineWeb-Edu~\citep{penedo2024finewebdatasetsdecantingweb} via a one-to-one ratio, while varying the total token budget across experiments. Since STE is a special case of \algacrofourier{}, we report the performance delta and training speed-up against DSQ. 

\paragraph{Results on LLaMA-3-1B.}
Across all evaluated bit-widths (2--4 bits), \algacro{} consistently outperforms both ParetoQ and DSQ (\autoref{tab:llama3-1b}), achieving improvements of up to \textcolor{ForestGreen}{\textbf{6.88\%}} under different training recipes.
At \textbf{4 bits}, \algacro{} delivers the best overall performance, surpassing ParetoQ and DSQ by \textcolor{ForestGreen}{\textbf{0.25\%}}--\textcolor{ForestGreen}{\textbf{1.38\%}}, and in several settings even exceeding the FP16 baseline.
The advantage becomes substantially more pronounced at \textbf{3 bits}, where \algacro{} consistently improves over both baselines by \textcolor{ForestGreen}{\textbf{2.70\%}}--\textcolor{ForestGreen}{\textbf{6.88\%}}, highlighting its effectiveness in the regime where QAT noise becomes severe. Under the most challenging \textbf{2-bit} setting, \algacro{} remains stable and achieves higher performance by \textcolor{ForestGreen}{\textbf{1.69\%}}, while ParetoQ and DSQ face more variance or training collapse (\autoref{fig:performance_error_bar_analysis}). Besides the standalone efficiency comparison in Setion~\ref{sec:efficiency_comparison}, \algacro{} achieves approximately \textcolor{ForestGreen}{\textbf{1.43$\times$}} end-to-end training speedup against DSQ with almost identical training cost to the standard STE-based QAT methods.



\paragraph{Results on LLaMA-3-3B.}
We observe consistent and often amplified trends on the larger LLaMA-3-3B model (\autoref{tab:llama3-3b}), indicating that the benefits of \algacro{} scale favorably with model size.
Across the \textbf{4-bit} and \textbf{3-bit} settings, \algacro{} uniformly outperforms both ParetoQ and DSQ, achieving gains of up to \textcolor{ForestGreen}{\textbf{2.67\%}} and \textcolor{ForestGreen}{\textbf{2.38\%}}, respectively.
Notably, the \textbf{4-bit} \algacro{} model surpasses the FP16 baseline, while the \textbf{3-bit} configuration reaches near full-precision performance, demonstrating that aggressive quantization can be achieved without sacrificing accuracy on larger models.
Under the most challenging \textbf{2-bit} setting, \algacro{} remains stable and continues to outperform ParetoQ, though its final accuracy is slightly below DSQ.
This suggests that while \algacro{} effectively stabilizes ultra-low-bit optimization, further improvements such as curriculum or staged training strategies may be beneficial to reach the peak performance, which we leave for future work. \algacro{} consistently achieves an end-to-end training speedup of approximately \textcolor{ForestGreen}{\textbf{1.27$\times$}} over DSQ.

\begin{table*}[t]
    \centering
    \caption{LLaMA-3.2-3B results. \algacro{} at 4-bit outperforms the 16-bit baseline, while 3-bit \algacro{} remains competitive.}
    \label{tab:llama3-3b}
    \resizebox{\linewidth}{!}{
    \begin{tabular}{ccccccccccccccc}
    \toprule
    \rowcolor{gray!20} Bits & Method & Setting & Arc-e & Arc-c & Boolq & Hellaswag & Openbookqa & Piqa & SciQ & Winogrande & Avg & $\Delta$ & GPU Days & SpeedUp\\
    \midrule
    16 & Baseline & Baseline & 71.63 & 45.99 & 73.39 &73.61 & 43.00 & 77.48 & 92.70 & 69.85 & 68.46 & -- & -- & --\\ 
    8 & Baseline no QAT & Baseline (no QAT) & 71.34 & 46.25 & 73.30 & 73.62 & 42.60 & 77.58 & 92.80 & 69.93 & 68.43 & -- & -- & -- \\ 
    6 & Baseline no QAT & Baseline (no QAT) & 71.89 & 45.65 & 73.98 & 73.57 & 42.00 & 77.20 & 92.80 & 70.09 & 68.40 & -- & -- & --\\
    \midrule
    4 & Baseline (no QAT) & Baseline (no QAT) & 61.99 & 37.63 & 68.44 & 66.87 & 36.40 & 74.48 & 89.40 & 61.96 & 62.15 & -- & -- & --\\ 
    4 & ParetoQ & 20B Tokens \& lr=1e-4 & 71.83 & 45.48 & 70.13 & 71.20 & 42.40 & 76.58 & 90.60 & 66.25 &  \underline{66.81} & -- & \textbf{5.77} & -- \\
    4 & DSQ & 20B Tokens \& lr=1e-4  &  70.19 & 41.73 & 68.39 & 64.58 & 39.40 &  76.51 & 90.93 & 64.40 &  64.48& -- & 7.29 & -- \\
     \rowcolor{blue!10}4 & \textbf{\algacro{}} & 20B Tokens \& lr=1e-4 & 72.05 & 44.97 & 70.21 & 71.28 & 42.00 & 77.58 & 91.50 & 67.64 & \textbf{67.15} & \textbf{\textcolor{green!60!black}{+2.67}} & 5.80 & \textbf{\textcolor{green!60!black}{1.26x}} \\
    \midrule
    3 & Baseline (no QAT) & Baseline (no QAT) & 26.14 & 25.43 & 44.01 & 26.61 & 28.40 & 52.50 & 26.00 & 49.01 & 34.76 & -- & -- & --\\
    3 & ParetoQ & 20B Tokens \& lr=1e-4& 71.47 & 45.16 & 70.28 & 70.00 & 42.00 & 76.15 & 91.00 & 65.93  & \underline{66.50} & -- & 5.97 & --\\
    3 & DSQ & 20B Tokens \& lr=1e-4  & 66.58 & 41.64 & 70.24 & 66.58 & 40.60 & 76.28 & 89.80 & 64.48 & 64.52 & -- & 7.43 & -- \\ 
    \rowcolor{blue!10}3& \textbf{\algacro{}} & 20B Tokens \& lr=1e-4 & 71.04 & 44.97 & 69.02 & 70.81 & 44.00 & 78.29 & 90.4 & 66.69 & \textbf{66.90} &  \textbf{\textcolor{green!60!black}{+2.38}} & \textbf{5.75} & \textbf{\textcolor{green!60!black}{1.29x}}\\ 
    \midrule
    2 & Baseline (no QAT) & Baseline (no QAT) & 25.55 & 24.91 & 43.88 & 26.03 & 28.80 & 53.26 & 21.40 & 48.30 & 34.02 & -- & -- & --\\
    2 & ParetoQ & 30B Tokens \& lr=1e-4  &  65.73 & 38.08 & 65.73 & 64.13 & 39.40 & 74.68 & 84.30 & 61.43 & 61.69 & -- & 8.81 & --\\
    2 & DSQ & 30B Tokens \& lr=1e-4 &  69.19 & 40.78 & 65.54 & 66.03 & 41.60 & 74.92 & 88.60 & 63.61 & \textbf{63.78} & -- & 11.08 & -- \\
    \rowcolor{blue!10} 2 & \textbf{\algacro{}} & 30B Tokens \& lr=1e-4  & 68.48 & 40.87 & 63.06 & 65.20 & 41.00  & 75.41 & 87.00 & 63.38 & \underline{63.05} & \textbf{\textcolor{red!60!black}{-0.73}} & \textbf{8.75} & \textbf{\textcolor{green!60!black}{1.27x}}\\
    \bottomrule
  \end{tabular}
  }
\end{table*}

\begin{figure}[ht]
    \centering

    \begin{subfigure}[t]{0.98\linewidth}
        \centering
        \includegraphics[width=0.49\linewidth]{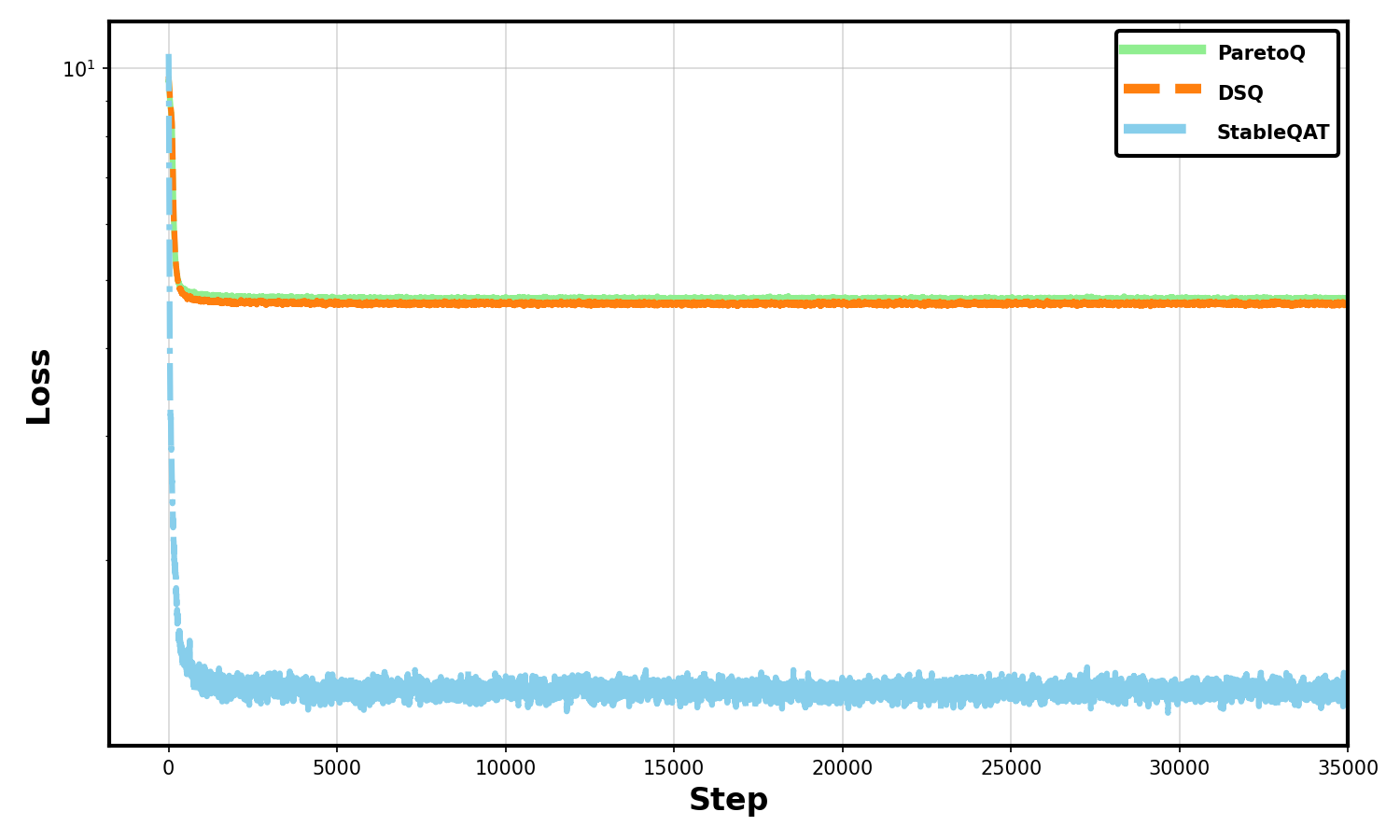}
        \includegraphics[width=0.49\linewidth]{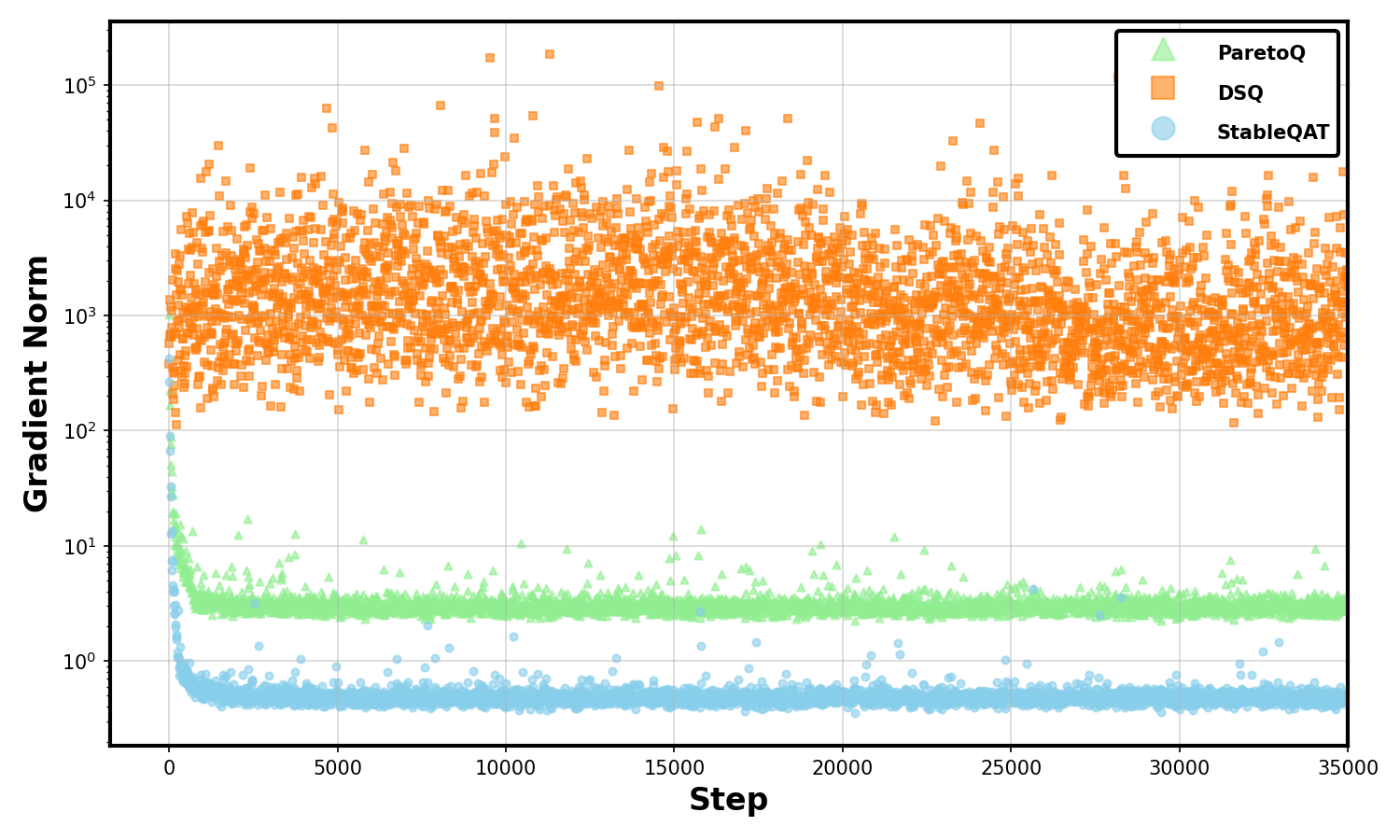}
        \caption{Learning rate $1\times 10^{-5}$.}
        \label{fig:train_loss_comparison_lr1e-5}
    \end{subfigure}
    \begin{subfigure}[t]{0.98\linewidth}
        \centering
        \includegraphics[width=0.48\linewidth]{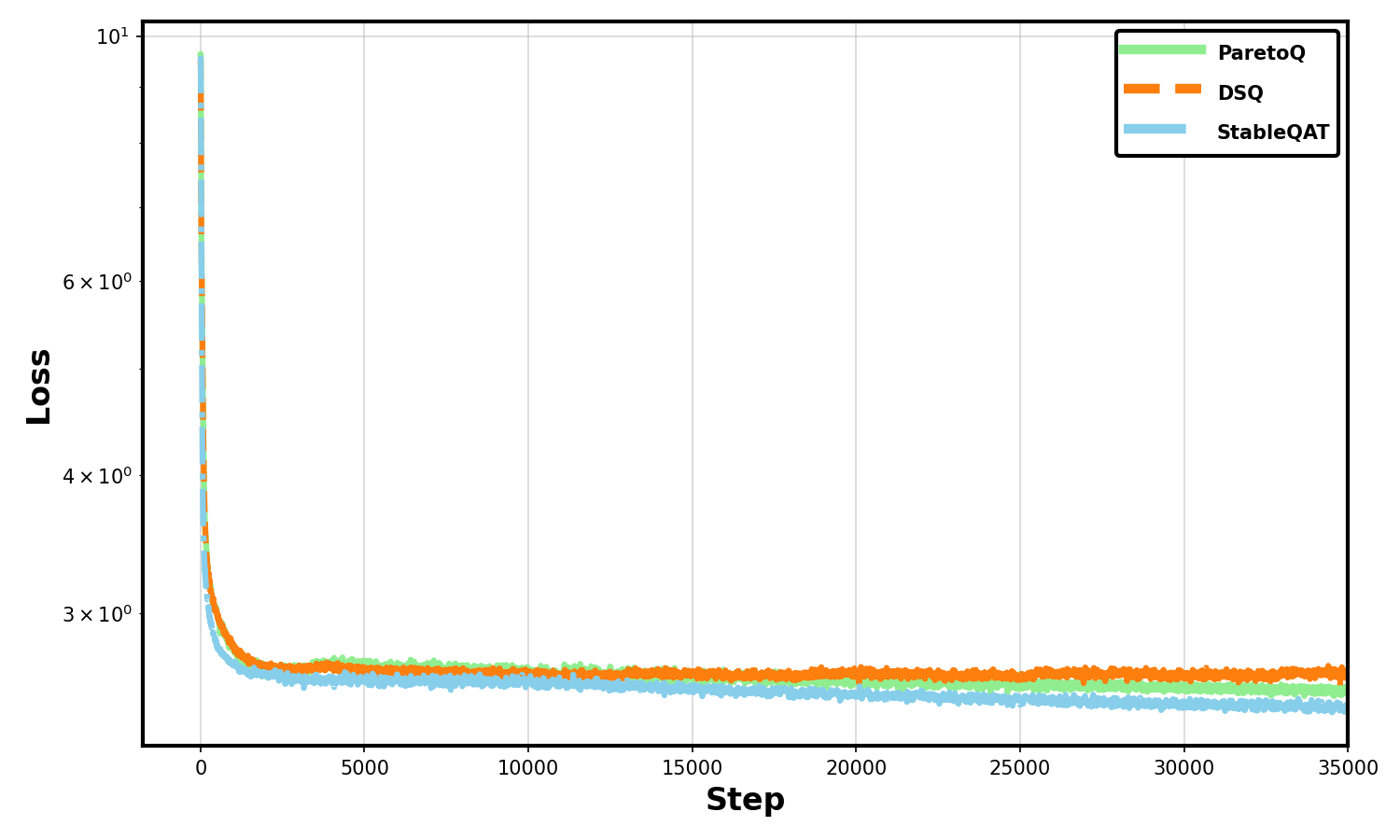}
        \includegraphics[width=0.48\linewidth]{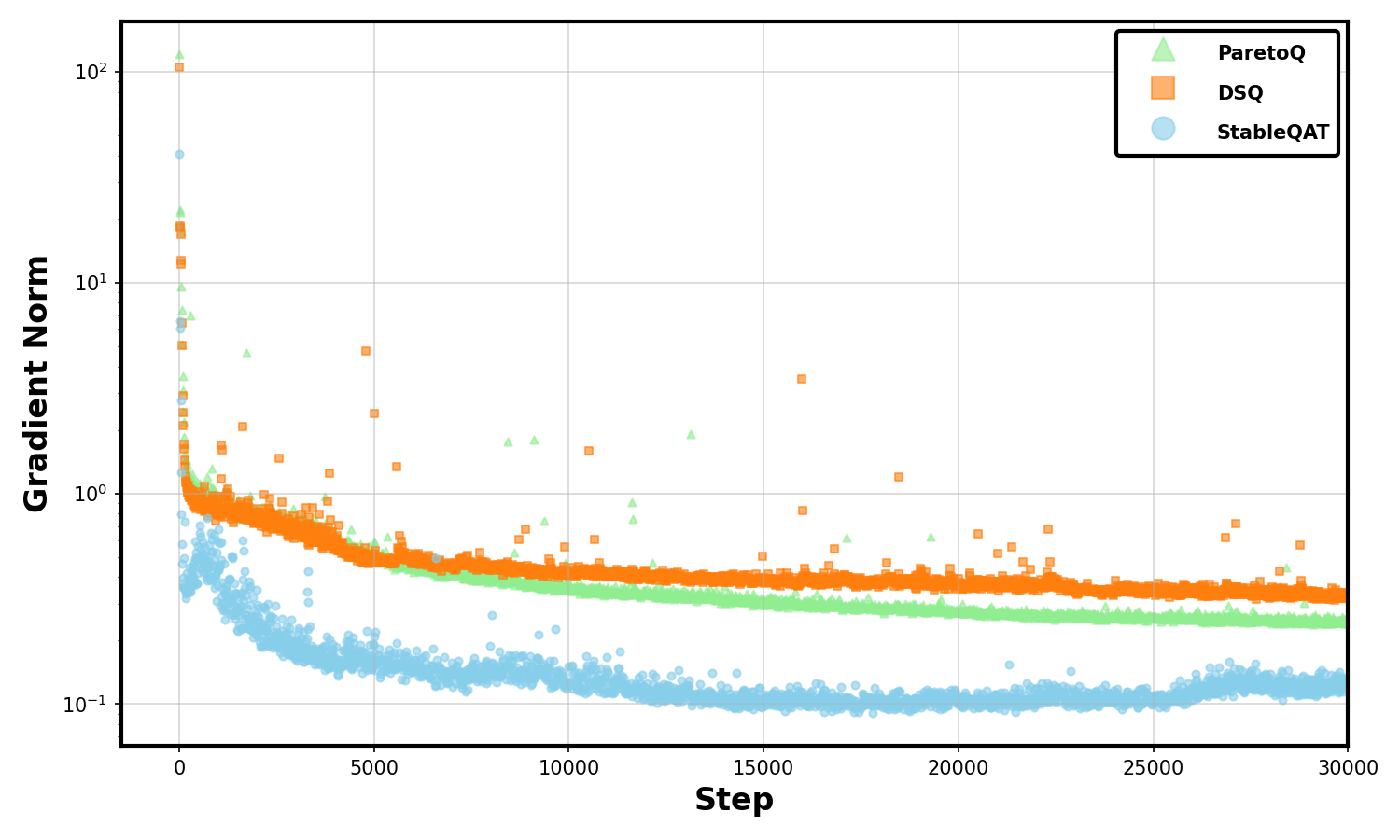}
        \caption{Learning rate $2\times 10^{-4}$.}
        \label{fig:train_loss_comparison_lr2e-4}
    \end{subfigure}
    \caption{Training loss (left) and gradient norm (right) comparison for Llama-3-1B under different learning rates.}
    \label{fig:train_loss_comparison}
\end{figure}

\begin{figure*}[h]
    \centering
    \begin{subfigure}[t]{0.32\linewidth}
        \centering
        \includegraphics[height=3.2cm]{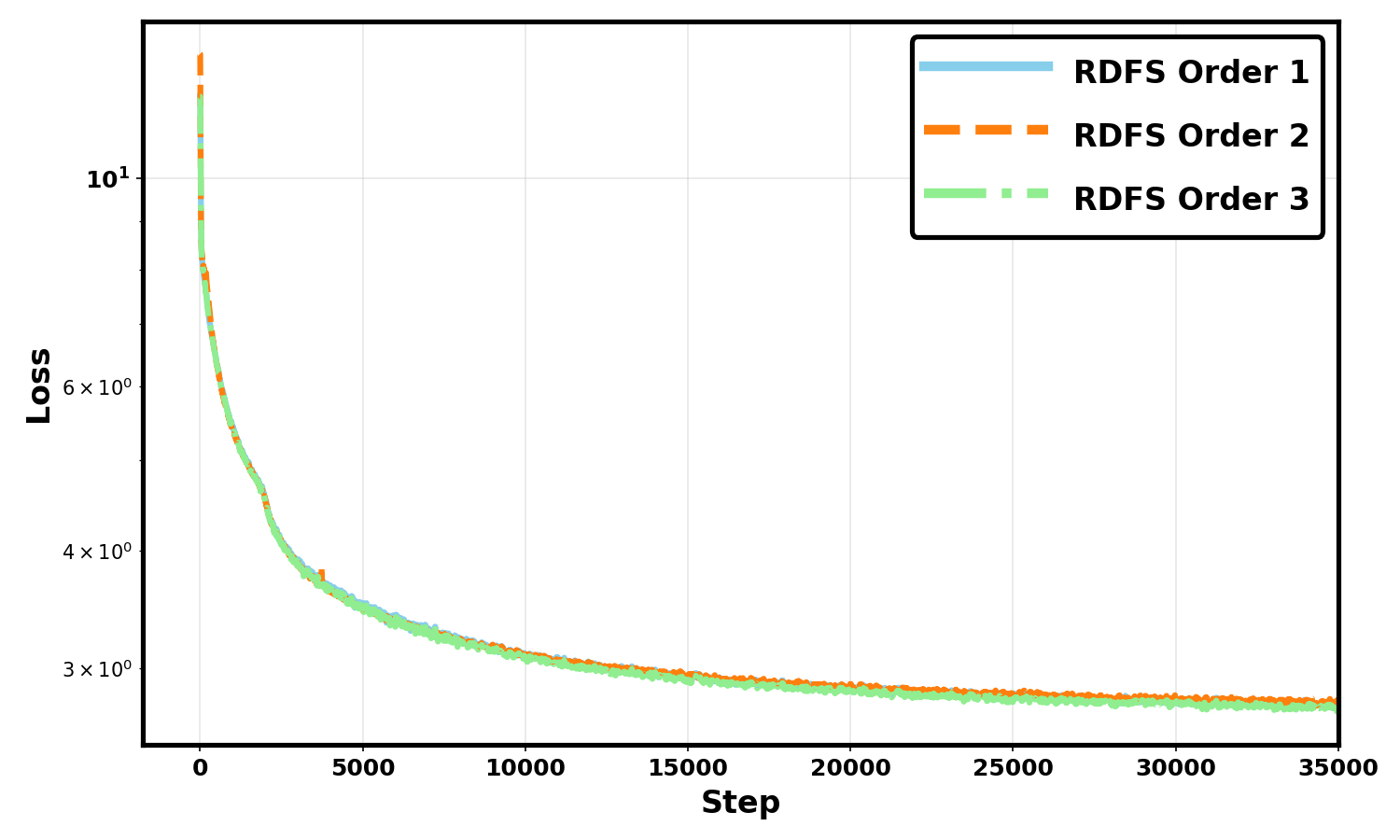}
        \caption{Training Loss}
        \label{fig:rdfs_loss}
    \end{subfigure}
    \begin{subfigure}[t]{0.32\linewidth}
        \centering
        \includegraphics[height=3.2cm]{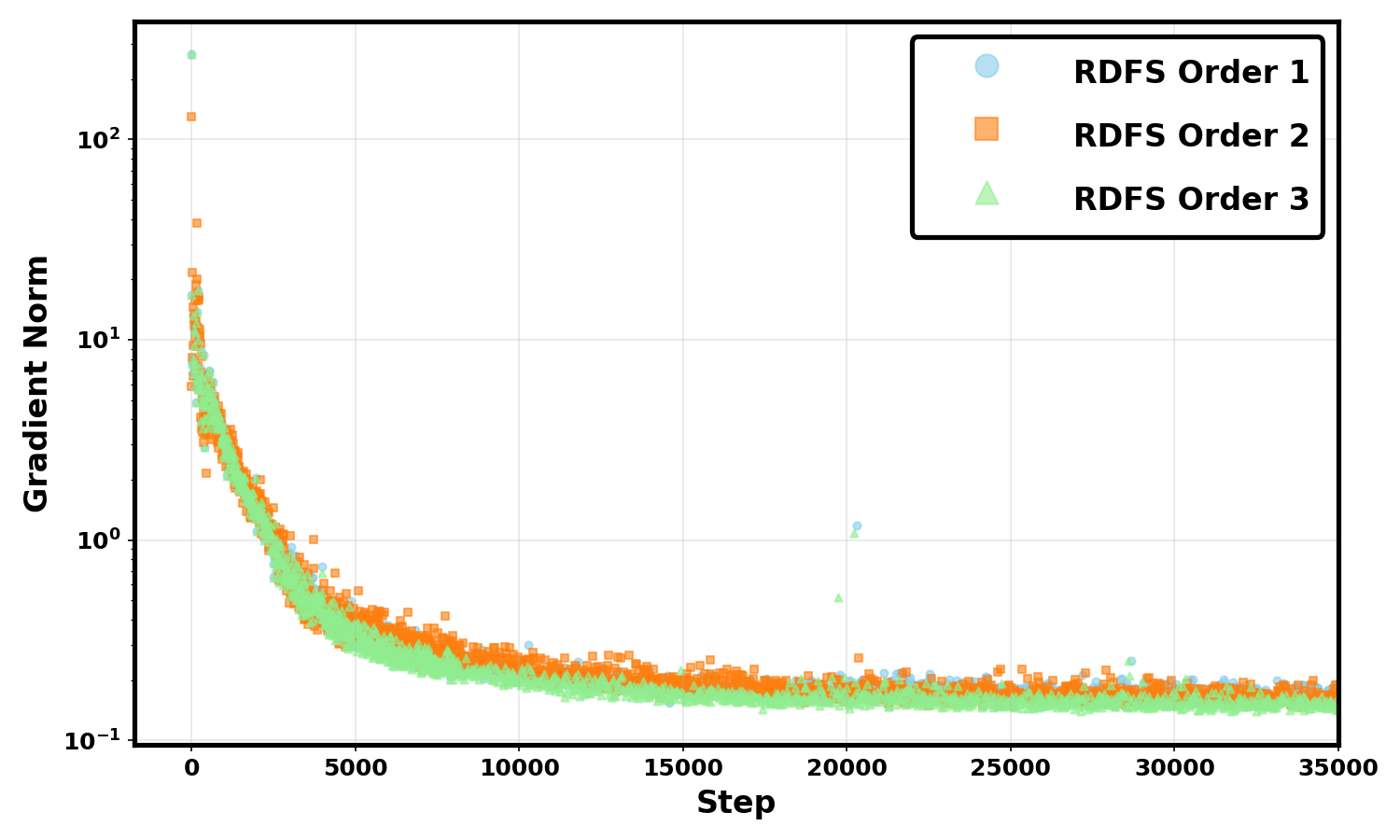}
        \caption{Gradient Norm}
        \label{fig:rdfs_gradnorm}
    \end{subfigure}
    \begin{subfigure}[t]{0.32\linewidth}
        \centering
        \includegraphics[height=3.2cm]{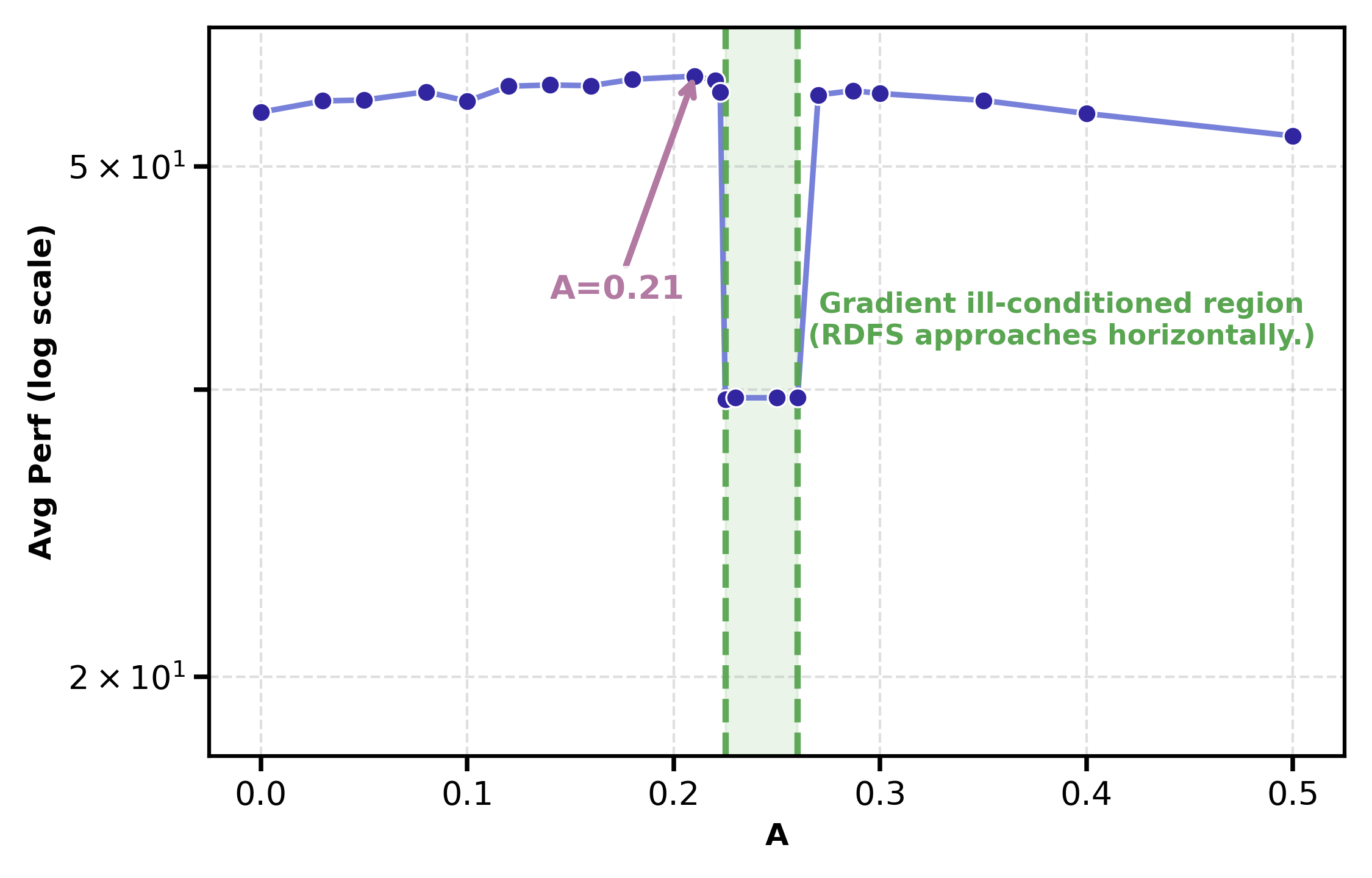}
        \caption{Amplitude Sensitivity}
        \label{fig:rdfs_amplitude}
    \end{subfigure}
    \caption{Training dynamics and amplitude ablation of RDFS.}
    \label{fig:rdfs_amp_orders_ablation}
\end{figure*}

\subsection{Training Stability Validation}

\paragraph{Gradient Dynamics and Convergence Behavior.} \autoref{fig:train_loss_comparison} provides a clear empirical validation of our theoretical analysis in Section~\ref{sec:theorem}. \algacro{} exhibits smooth and well-behaved optimization dynamics across learning rates, with steadily decreasing training loss and controlled gradient norms throughout training. This behavior is consistent with \autoref{thm:best_L2_ste}, which shows that our rotated damped Fourier surrogate achieves strictly smaller approximation error to the rotated rounding function than STE, resulting in a more faithful optimization direction and improved convergence. Moreover, \algacro{} maintains reliable and substantial gradient signals while effectively eliminating extreme gradient outliers. In contrast to DSQ, which typically display sharp gradient spikes and outliers inducing instability. The phenomenon is well aligned with~\autoref{thm:asymptotic_dsq}, the gradient variance of \algacro{} remains bounded, leading to gradual gradient-norm decay rather than explosion or premature vanishing. Together, these properties enable \algacro{} to converge to better optimum, demonstrating that its theoretical advantages translate directly into stable and reliable training behavior in practice.

\begin{wrapfigure}{r}{0.48\linewidth}
    \centering
    \vspace{-6pt}
    \hspace{-8pt}
    \includegraphics[width=\linewidth]{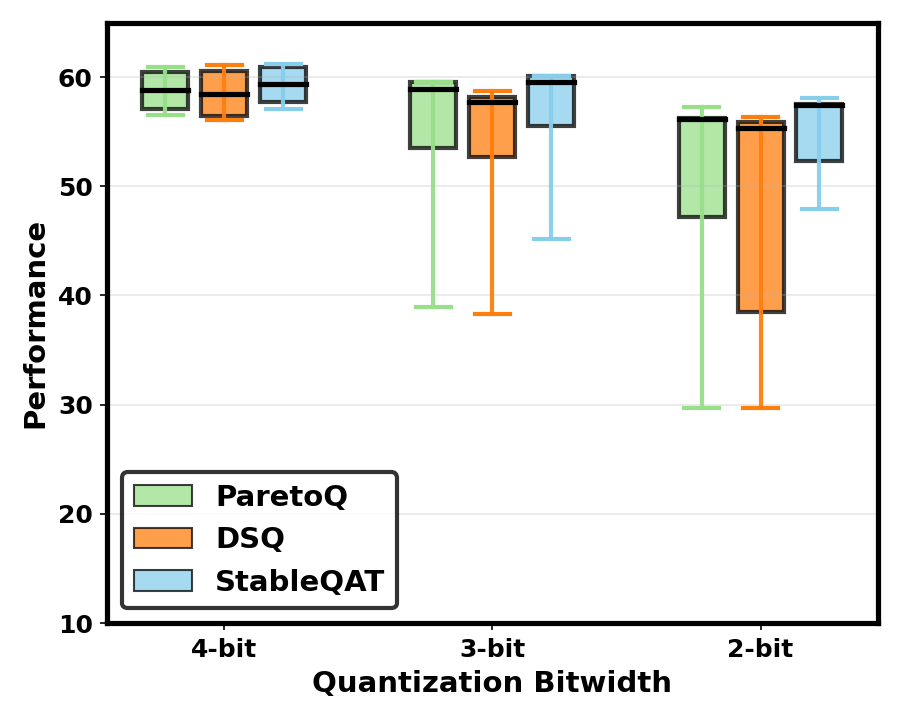}
    \vspace{-0.5em}
    \caption{Performance error bar for LLaMA-3-1B.}
    \label{fig:performance_error_bar_analysis}
    \vspace{-8pt}
\end{wrapfigure}
\paragraph{Performance Error Bar.} 
The performance error bars in~\autoref{fig:performance_error_bar_analysis}  characterize training robustness across multiple random seeds and learning-rate settings. For each method and bit-width, we report the mean performance together with its dispersion, thereby capturing sensitivity to optimization noise beyond single-run best results. 
\algacro{} consistently achieves the highest average performance with the tightest error bars, indicating reliable convergence under different hyperparameters, with the margin becoming more pronounced at lower bit-widths where optimization is particularly fragile. In contrast, ParetoQ exhibits both lower mean performance and larger variance. Its relies on the STE, which does not capture the intrinsic structure of the rounding operator, often yields noisy and misleading update directions. These misaligned gradients can push weights across quantization thresholds in an uncontrolled manner, increasing sensitivity, sometimes leading to training collapse. DSQ shows elevated variance for a different reason. Its sigmoid-style surrogate introduces regions of extremely large gradients near the transition boundaries. Such gradient explosion amplifies small perturbations during training, making optimization sensitive and introducing widened error bars.

\subsection{Ablations}

\paragraph{Fourier Order of \algacrofourier{}.}
We study the effect of Fourier truncation order $M$ in \algacrofourier{}. \autoref{fig:rdfs_loss} and \autoref{fig:rdfs_gradnorm} compare training dynamics across $M=0,1,2$. A first-order approximation ($M=0$) already captures the majority of performance and stability gains, yielding a smooth loss decay along with low and well-controlled gradient norms. While higher-order terms introduce finer structural details of the quantization operator, their empirical benefits are marginal and come with increased numerical complexity. These observations validate our design choice of adopting the first-order \algacrofourier{} as \autoref{eq:rdfs_first_order}, which offers an effective balance between approximation fidelity, training stability, and computational simplicity.

\paragraph{Amplitude.} We next vary the amplitude $A$ to control the sharpness of  surrogate gradient, to empirically validate the damped amplitude design predicted by the theoretical analysis in Section~\ref{sec:amplitude}. As shown in~\autoref{fig:rdfs_amplitude}, the empirical results closely align with the theoretical characterization and clearly reveal the presence of an ill-conditioned regime, where the surrogate curves become nearly tangential to horizontal plateaus.  In contrast, moderate amplitudes provide a favorable balance between approximation fidelity and optimization stability. The selected default setting $A = 0.21$ achieves strong performance, benefiting from an adequate approximation of rounding operator while maintaining stable and well-conditioned training dynamics.

\section{Conclusion}

We propose \algacro{} to address the training instability and suboptimal performance of quantization-aware training (QAT). \algacro{} employs a novel rotated damped Fourier surrogate (\algacrofourier{}) that captures the intrinsic structure of quantization while avoiding the gradient instability of conventional soft-quantization. We provide theoretical guarantees showing that \algacrofourier{} ensures bounded gradient variance and preserves admissible gradient magnitudes, enabling stable and effective optimization. Extensive experiments on large language models (LLMs) and vision transformers (ViTs) demonstrate consistent improvements in both peak and average performance under low-bit QAT.

\bibliography{reference}
\bibliographystyle{icml2026}

\appendix

\newpage 
\onecolumn

\section{More Related Works: Complementary but Orthogonal Views to \algacro{}}\label{appendix:more_related_works}

Recent works on post-training quantization (PTQ) focus on mitigating outliers and distribution skewness, which become increasingly problematic at lower bitwidths. Techniques such as function-preserving transformations and channel-wise clustering aim to reshape weight distributions so that uniform or centroid-based quantizers incur less error without modifying model parameters \cite{fptquant2025, skim2025,gu2025ultra}. QuaRot~\citep{ashkboos2024quarot}, SpinQuant~\citep{liu2024spinquant} assume that quantization error can be mitigated via orthogonal rotations that redistribute variance across dimensions. However, at ultra-low bitwidths, the dominant source of error shifts from distributional anisotropy to discretization-induced geometry collapse. The quantized space forms an extremely coarse discrete lattice in which many distinct directions become indistinguishable. Orthogonal rotations can redistribute variance but cannot increase the intrinsic angular resolution of the quantized space, and therefore cannot recover representational geometry lost through discretization. AWQ~\citep{lin2024awq} allocates quantization precision according to activation saliency by rescaling weight channels. This strategy better preserves semantically important subspaces and thus improves performance at moderate bitwidths. However, AWQ remains fundamentally limited by the discrete resolution of the quantized space: while it redistributes precision across channels, it cannot increase the number of representable states. Consequently, at ultra-low bitwidths, discretization-induced geometry collapse dominates, and channel-wise rescaling becomes insufficient to preserve fine-grained angular relationships in attention and representation spaces.

A complementary line of work mitigates attention outliers through architectural reformulation rather than weight rotation. Outlier-efficient attention mechanisms such as Softmax1-based formulations reinterpret attention from an associative memory perspective and suppress low-information tokens by introducing refined energy functions \cite{hu2024outlier}. Subsequent extensions apply similar mechanisms to improve robustness by replacing vanilla Softmax with outlier-efficient alternatives \cite{luo2025fast, luo2026frost}. While such architectural interventions can improve robustness to post-training quantization, they do not directly address the discretization-induced geometry collapse that arises when the representational lattice itself becomes too coarse. In contrast, \algacro{} focuses on stabilizing quantization-aware-training optimization under low-bit constraints by reformulating the rounding surrogate, thereby targeting the training-time gradient geometry rather than attention-level distribution shaping.

While most PTQ methods aim to approximate full precision model as accurately as possible, QAT instead tries to learn a new model under low-bit constraints. However, as bitwidth decreases, STE-based training often becomes unstable or biased. Several works focus on improving training efficiency and scalability of QAT rather than altering the quantization operator itself; examples include block-wise or staged training schemes that reduce computational cost, as well as data-free distillation strategies that remove dependence on large calibration datasets \cite{llmqat2024, efficientqat2025}. Another thread leverages self-distillation objectives to stabilize sub-4-bit settings \cite{bitdistiller2024}. 

A complementary line of work reformulates quantization as an approximation or decomposition problem. For instance, QBB represents weights as linear combinations of binary bases, replacing multiplications with additions while optimizing scales via distillation \cite{qbb2024,kim2023teacher}. LoRA-aware approaches such as LoftQ and IR-QLoRA integrate quantization with low-rank adaptation, either by alternating between quantization and decomposition or by explicitly preserving information through auxiliary connections \cite{loftq2024, irqlora2024}. Another body of work treats quantization through stochastic or noise-based perspectives, aiming to avoid hard discrete operators during training such as DiffQ which introduces pseudo quantization noise to approximate quantization effects \cite{diffq2022}. Similarly, stochastic differentiable quantization frameworks have bitwidths sampled probabilistically and optimized via reparameterization \cite{sdq2022}.

Differentiable and soft quantization methods more directly target the rounding operation by introducing continuous relaxations. DSQ is a representative example, utilizing parameterized smooth function whose sharpness increases during training \cite{gong2019dsq}. DSMM \cite{dsmm2024} introduces a differentiable soft min–max loss to reduce weight ranges by making max/min approximations differentiable and learnable.

Recent work pushes differentiable quantization further by explicitly optimizing bitwidths or noise scales as continuous variables. GDNSQ \cite{gdnsq2025} frames quantization as a smooth constrained optimization over effective bitwidths, combining STE-style estimators with gradual noise scheduling and distillation losses. Floating-point oriented frameworks \cite{difffpq2025} extend pseudo-quantization training to mixed floating-point formats, enabling differentiable optimization of mantissa and exponent precisions. Other works provide scaling laws and diagnostics to assess failure modes at low-bit regimes \cite{paretoq2025, beyondoutliers2025}. However, a key challenge in low-bit QAT that remains lies in reconciling discrete inference with stable, informative gradients.

\section{Additional Experimental Results on Vision Transformer}\label{appendix:vit}

\begin{table}[h]
\centering
\caption{VIT on ImageNet-1K.
``Bit-width (W/A)'' denotes the bitwidth for weights and activations.}
\label{tab:imagenet_qat}
\setlength{\tabcolsep}{3pt}
\scriptsize
\begin{tabular}{l | l | cc cc cc}
\toprule
Network & QAT Method &
\makecell[c]{Bit-width \\ (W/A)} & Top-1 &
\makecell[c]{Bit-width \\ (W/A)} & Top-1 &
\makecell[c]{Bit-width \\ (W/A)} & Top-1 \\
\midrule

\multirow{5}{*}{\makecell[l]{DeiT-T \\ (FP Top-1: 73.8)}}
& Q-ViT       & 4/4$^\dagger$ & \underline{72.79} & 3/3$^\dagger$ & \underline{69.62} & -- & -- \\
& GPUSQ-ViT    & 4/4            & 71.70 & --             & --    & -- & -- \\
& PackQViT   & 4/4            & 72.70 & --             & --    & -- & -- \\
& LSQ+                  &  4/4            & 72.62 & 3/3            & 68.22 & 2/2 & \underline{54.45} \\
& \cellcolor{gray!15}\textbf{\algacro{}} & \cellcolor{gray!15}4/4 & \cellcolor{gray!15}\textbf{73.08} & \cellcolor{gray!15}3/3 & \cellcolor{gray!15}\textbf{69.80} & \cellcolor{gray!15}2/2 & \cellcolor{gray!15}\textbf{56.31} \\
\midrule


\multirow{6}{*}{\makecell[l]{Swin-T \\ (FP Top-1: 81.0)}}
& Q-ViT       &  4/4$^\dagger$ & 80.59 & 3/3$^\dagger$ & 79.45 & -- & -- \\
& GPUSQ-ViT    &  4/4            & 80.70 & --             & --    & -- & -- \\
& PackQViT   &  4/4            & 81.50 & --             & --    & -- & -- \\
& Li et al.  &  4/4            & \textbf{82.10} & 3/3            & \underline{80.57} & 2/2 & \underline{74.31} \\
& LSQ+                  &  4/4            & 80.61 & 3/3            & 79.07 & 2/2 & 70.21 \\
& \cellcolor{gray!15}\textbf{\algacro{}} &  \cellcolor{gray!15}{4/4} & \cellcolor{gray!15}\underline{81.98} & \cellcolor{gray!15}{3/3} & \cellcolor{gray!15}\textbf{81.02} & \cellcolor{gray!15}{2/2} & \cellcolor{gray!15}\textbf{75.86} \\
\bottomrule
\end{tabular}
\vspace{0.5em}
\footnotesize\\
$^\dagger$ Average bitwidth for mixed-precision quantization.
\end{table}

To demonstrate the generality of \algacro{} beyond LLMs, we further evaluate it on representative Vision Transformer architectures. Following prior QAT studies~\citep{huang2023quantization,qu2025automatic}, we consider DeiT~\citep{touvron2021training} and Swin~\citep{liu2021swin}. We compare \algacro{} against several state-of-the-art ViT quantization methods, including Q-ViT~\citep{li2022q}, LSQ+~\citep{bhalgat2020lsqplus}, PackQViT~\citep{dong2023packqvit} and GPUSQ-ViT~\citep{yu2023boost}, on the ImageNet-1k benchmark~\citep{deng2009imagenet}. The results are present in \autoref{tab:imagenet_qat}. We adopted \algacro{} upon the code-base of \citep{huang2023quantization} while exclude knowledge distillation for fair comparison. Numerical results present that \algacro{} consistently brings performance again across varying bits and ViTs, demonstrating its generality and transferability across different applications and architectures.

\section{Rotated Damped Fourier Surrogate Derivation}\label{appendix:fourier_rounding}

\paragraph{Fourier series derivation.} 
We first derive the Fourier series expansion of the zig-zag function~\eqref{eq:rotated_round_func} written as
\begin{equation*}
f(t)=\frac{1}{2\sqrt{2}} \left(1-4\left| r(t)-\frac12 \right|\right),
\end{equation*}
where $r(t) = \left\{\frac{t - T/4}{T}\right\}$, $T = \sqrt{2}$ is the period, and $\{\cdot\}$ denotes the fractional part function. Since $f(t)$ is square integrable and periodic, it admits a Fourier series~\cite{stein2011fourier}:
\begin{equation*}
f(t) = a_0 + \sum_{k=1}^{\infty} \left[a_k \cos\!\left(\frac{2\pi k t}{T}\right) + b_k \sin\!\left(\frac{2\pi k t}{T}\right)\right].
\end{equation*}
Due to the phase shift by $T/4$ in the definition of $r(t)$, the function $f(t)$ is odd, i.e., $f(-t) = -f(t)$. Consequently, $a_k = 0$ for all $k\in \mathbb{N}$, leaving only sine terms:
\begin{equation*}
f(t)=\sum_{k=1}^{\infty} b_k \sin\!\left(\frac{2\pi k}{T}t\right).
\end{equation*}
The Fourier coefficients, computed via $b_k = \frac{2}{T}\int_0^T f(t)\,\sin\!\left(\frac{2\pi k t}{T}\right) dt$, are given by for all $m\in \mathbb{N}$,
\begin{equation*}
b_{2m+1} = -\frac{2\sqrt{2}}{\pi^2} \cdot \frac{(-1)^m}{(2m+1)^2}, \quad b_{2m+2} = 0.
\end{equation*}
Thus, only odd harmonics contribute, and the zig-zag function admits the sine-series expansion
\begin{equation*}
f(t) = -\frac{2\sqrt{2}}{\pi^2} \sum_{m=0}^{\infty} \frac{(-1)^m}{(2m+1)^2} \sin\bigl((2m+1)\sqrt{2}\pi\,t\bigr).
\end{equation*}
To obtain a tunable surrogate, we replace the fixed amplitude $\frac{2\sqrt{2}}{\pi^2}$ with a learnable parameter $A$:
\begin{equation*}
f(t) = -A \sum_{m=0}^{\infty} \frac{(-1)^m}{(2m+1)^2} \sin\bigl((2m+1)\sqrt{2}\pi\,t\bigr),
\end{equation*}
with derivative
\begin{equation}\label{eq:fourier-series-derivative}
f'(t) = -A\sqrt{2}\pi \sum_{m=0}^{\infty} \frac{(-1)^m}{2m+1} \cos\bigl((2m+1)\sqrt{2}\pi\,t\bigr).
\end{equation}

\paragraph{Rotation to the rounding function.} 
The key observation is that rotating the zig-zag function by $45^\circ$ counterclockwise yields the staircase rounding function. To see this, we apply the standard rotation matrix to each point $(t, f(t))$ on the graph:
\begin{equation}\label{eq:parameterized-curve}
\begin{pmatrix} x \\ x_q \end{pmatrix}
= \frac{1}{\sqrt{2}}
\begin{pmatrix} 1 & -1 \\ 1 & 1 \end{pmatrix}
\begin{pmatrix} t \\ f(t) \end{pmatrix}
= \left( \frac{t - f(t)}{\sqrt{2}},\; \frac{t + f(t)}{\sqrt{2}} \right)^T,
\end{equation}
where $x$ represents the input and $x_q$ represents the quantized (rounded) output. Inverting this relationship gives $t = (x + x_q)/\sqrt{2}$.

\paragraph{Gradient surrogate derivation.} 
Using the parameterization in Equation~\eqref{eq:parameterized-curve} and the chain rule, we have
\begin{equation}\label{eq:parameterized-curve-derivative}
\frac{\partial x_q}{\partial x}
= \frac{\partial x_q / \partial t}{\partial x / \partial t}
= \frac{1 + f'(t)}{1 - f'(t)}.
\end{equation}
Substituting the Fourier series derivative from Equation~\eqref{eq:fourier-series-derivative} and using $t = (x + x_q)/\sqrt{2}$, we obtain
\begin{equation*}
\frac{\partial x_q}{\partial x}
= \frac{
1 - A\sqrt{2}\pi \displaystyle\sum_{m=0}^{\infty} \frac{(-1)^m}{2m+1} \cos\bigl((2m+1)\pi (x + x_q)\bigr)
}{
1 + A\sqrt{2}\pi \displaystyle\sum_{m=0}^{\infty} \frac{(-1)^m}{2m+1} \cos\bigl((2m+1)\pi (x + x_q)\bigr)
}.
\end{equation*}

\section{Proof of Theorem~\ref{thm:best_L2_ste}}\label{appendix:thm5_1}
For clarity in the proof, we restate Theorem~\ref{thm:best_L2_ste} using precise mathematical notation.
\begin{theorem}
Let $f \in L^2([0,T])$ and let $f_n$ be the $n$th partial Fourier sum of the function $f$. Define the function space of trigonometric polynomials of degree at most $n$ by
\begin{equation*}
\Gcal_n = \operatorname{span}\!\left\{
1,\,
\cos\!\left(\frac{2\pi k u}{T}\right),\,
\sin\!\left(\frac{2\pi k u}{T}\right)
\,:\,
1 \le k \le n
\right\}
\end{equation*}
(i) $f_n$ is the unique minimizer of the $L^2$ approximation error over the space $\Gcal_n$, i.e., 
\begin{equation*}
f_n=\operatorname*{arg\,min}_{g \in \Gcal_n}
\|f - g\|_{L^2}.
\end{equation*}
(ii) For all $n \in \mathbb{N}\backslash \{0\}$, the strict inequality
\begin{equation*}
\|f - f_n\|_{L^2} < \|f - f_0\|_{L^2}
\end{equation*}
holds if and only if $f$ is not equal to constant almost everywhere on $[0,T]$.
\end{theorem}
\begin{proof}
We first prove part (i). The Fourier coefficients of $f$ on the interval $[0,T]$ are defined by
\begin{equation}\label{eq:fourier_inner_products}
\begin{aligned}
\alpha_0 &= \frac{1}{T} \int_0^T f(u)\,du, \\
\alpha_k &= \frac{2}{T} \int_0^T
f(u)\cos\!\left(\frac{2\pi k u}{T}\right)du, \text{for } k\geq 1, \\
\qquad
\beta_k &= \frac{2}{T} \int_0^T f(u)\sin\!\left(\frac{2\pi k u}{T}\right)du, \text{for } k\geq 1.
\end{aligned}
\end{equation}

Note that we have, for $j, k \geq 1$,
\begin{equation}\label{eq:sin_cos_relation}
\int_0^T \cos\!\left(\frac{2\pi j u}{T}\right)\cos\!\left(\frac{2\pi k u}{T}\right)du = \frac{T}{2}\delta_{jk}, \quad 
\int_0^T \sin\!\left(\frac{2\pi j u}{T}\right)\sin\!\left(\frac{2\pi k u}{T}\right)du = \frac{T}{2}\delta_{jk}, \quad 
\int_0^T \cos\!\left(\frac{2\pi j u}{T}\right)\sin\!\left(\frac{2\pi k u}{T}\right)du = 0,
\end{equation}
where $\delta_{jk}$ is the Kronecker delta and it is equal to $1$ if $j=k$ and $0$ otherwise.

Now let $g \in \Gcal_n$ be an arbitrary trigonometric polynomial of the form
\[
g(u) = a_0 + \sum_{k=1}^n \left[a_k \cos\!\left(\frac{2\pi k u}{T}\right) + b_k \sin\!\left(\frac{2\pi k u}{T}\right)\right].
\]
Expanding the squared $L^2$ norm error $\|f - g\|_{L^2}^2$, we have that
\begin{equation}\label{eq:error_expansion}
\|f - g\|_{L^2}^2 = \|f\|_{L^2}^2 - 2\langle f, g \rangle_{L^2} + \|g\|_{L^2}^2.
\end{equation} 
For the inner product $\langle f, g \rangle_{L^2}$ in~\eqref{eq:error_expansion}, we have from~\eqref{eq:fourier_inner_products} that
\begin{equation}\label{eq:inner_product_f_g}
\begin{aligned}
\langle f, g \rangle_{L^2} &= \int_0^T f(u) \left[a_0 + \sum_{k=1}^n \left(a_k \cos\!\left(\frac{2\pi k u}{T}\right) + b_k \sin\!\left(\frac{2\pi k u}{T}\right)\right)\right] du \\
&= a_0 \int_0^T f(u)\, du + \sum_{k=1}^n \left[a_k \int_0^T f(u) \cos\!\left(\frac{2\pi k u}{T}\right) du + b_k \int_0^T f(u) \sin\!\left(\frac{2\pi k u}{T}\right) du\right] \\
&= T\alpha_0 a_0 + \frac{T}{2}\sum_{k=1}^n (\alpha_k a_k + \beta_k b_k). 
\end{aligned}
\end{equation}
For the squared norm $\|g\|_{L^2}^2$ in~\eqref{eq:error_expansion}, we have from~\eqref{eq:sin_cos_relation} that
\begin{equation}\label{eq:g_norm_squared}
\begin{aligned}
\|g\|_{L^2}^2 =& \int_0^T \left[a_0 + \sum_{k=1}^n \left(a_k \cos\!\left(\frac{2\pi k u}{T}\right) + b_k \sin\!\left(\frac{2\pi k u}{T}\right)\right)\right]^2 du. \\
= &a_0^2 \cdot T + \sum_{k=1}^n \left[a_k^2 \cdot \frac{T}{2} + b_k^2 \cdot \frac{T}{2}\right] = Ta_0^2 + \frac{T}{2}\sum_{k=1}^n (a_k^2 + b_k^2)
\end{aligned}
\end{equation}
It then follows from~\eqref{eq:error_expansion},~\eqref{eq:inner_product_f_g} and~\eqref{eq:g_norm_squared} that
\begin{equation}\label{eq:error_formula}
\begin{aligned}
\|f - g\|_{L^2}^2 &= \|f\|_{L^2}^2 - 2\left[T\alpha_0 a_0 + \frac{T}{2}\sum_{k=1}^n (\alpha_k a_k + \beta_k b_k)\right] + Ta_0^2 + \frac{T}{2}\sum_{k=1}^n (a_k^2 + b_k^2) \\
&= \|f\|_{L^2}^2 + T(a_0^2 - 2\alpha_0 a_0) + \frac{T}{2}\sum_{k=1}^n \left[(a_k^2 - 2\alpha_k a_k) + (b_k^2 - 2\beta_k b_k)\right] \\
&= \|f\|_{L^2}^2 - T\alpha_0^2 - \frac{T}{2}\sum_{k=1}^n (\alpha_k^2 + \beta_k^2) + T(\alpha_0 - a_0)^2 + \frac{T}{2}\sum_{k=1}^n \left[(\alpha_k - a_k)^2 + (\beta_k - b_k)^2\right]. 
\end{aligned}
\end{equation}
Recall that the $n$th partial Fourier sum is
\begin{equation}\label{eq:partial_fourier_sum}
f_n(u) = \alpha_0 + \sum_{k=1}^n \left[\alpha_k \cos\!\left(\frac{2\pi k u}{T}\right) + \beta_k \sin\!\left(\frac{2\pi k u}{T}\right)\right]. 
\end{equation}
Replacing $g$ in $\|f - g\|_{L^2}^2$ with $f_n$ given in~\eqref{eq:partial_fourier_sum}, we have that
\begin{equation}\label{eq:L2_error_fn}
\|f - f_n\|_{L^2}^2 = \|f\|_{L^2}^2 - T\alpha_0^2 - \frac{T}{2}\sum_{k=1}^n (\alpha_k^2 + \beta_k^2).
\end{equation}
Substituting~\eqref{eq:L2_error_fn} back into~\eqref{eq:error_formula}, we obtain
\begin{equation}\label{eq:final_error_form}
\|f - g\|_{L^2}^2 = \|f - f_n\|_{L^2}^2 + T(\alpha_0 - a_0)^2 + \frac{T}{2}\sum_{k=1}^n \left[(\alpha_k - a_k)^2 + (\beta_k - b_k)^2\right].
\end{equation}
Since all terms on the right-hand side of~\eqref{eq:final_error_form} beyond $\|f - f_n\|_{L^2}^2$ are non-negative, it follows that $\|f - g\|_{L^2}^2 \geq \|f - f_n\|_{L^2}^2$ for all $g \in \mathcal{G}_n$. Moreover, equality holds if and only if $\alpha_0 = a_0$, $\alpha_k = a_k$ and $\beta_k = b_k$ for all $k \in \{1, \cdots, n\}$. This establishes that $f_n$ is the unique minimizer of the $L^2$ approximation error over $\mathcal{G}_n$, which completes the first part of the proof.

We next prove part (ii). In the first part of the proof, we have shown that
\[
\|f - f_n\|_{L^2}^2 = \|f\|_{L^2}^2 - T\alpha_0^2 - \frac{T}{2}\sum_{k=1}^n (\alpha_k^2 + \beta_k^2).
\]
In particular, for $n = 0$, this gives $\|f - f_0\|_{L^2}^2 = \|f\|_{L^2}^2 - Ta_0^2$. Taking the difference between $\|f - f_0\|_{L^2}^2$ and $\|f - f_n\|_{L^2}^2$, we obtain
\begin{equation}\label{eq:difference}
\|f - f_0\|_{L^2}^2 - \|f - f_N\|_{L^2}^2 = \frac{T}{2}\sum_{k=1}^N (\alpha_k^2 + \beta_k^2).
\end{equation}
Since each term at the right-hand side of equation~\eqref{eq:difference} is non-negative, we have $\|f - f_N\|_{L^2} \leq \|f - f_0\|_{L^2}$ for all $N \geq 1$, with strict inequality if and only if at least one of $\alpha_1, \beta_1, \ldots, \alpha_n, \beta_n$ is non-zero. 

It then remains to show that at least one of $\alpha_1, \beta_1, \ldots, \alpha_n, \beta_n$ is non-zero is equivalent to the statement that $f$ is not equal to constant almost everywhere on $[0,T]$. We will prove the contrapositive statement. First recall that the constant functions on $[0,T]$ form the one-dimensional subspace
\[
\operatorname{span}\{1\} \subset L^2([0,T]).
\]
Suppose that $f$ is equal to a constant almost everywhere on $[0,T]$. Since $\cos(2\pi k u/T)$ and $\sin(2\pi k u/T)$ are orthogonal to constant functions in the space $L^2([0,T])$, $\alpha_k = \beta_k = 0$ for all $k \geq 1$. It then follows that $f(u)=\alpha_0$ almost everywhere. Conversely, suppose that for all $k \geq 1$, $\alpha_k = \beta_k = 0$. Then the Fourier series of $f$ reduces to the constant term $\alpha_0$, and the partial Fourier sums satisfy that for all $n \in \mathbb{N}$,
\[
f_n(u) = \alpha_0.
\]
Since $f_n \to f \in L^2([0,T])$, it follows that $f = \alpha_0$ almost everywhere on $[0,T]$. Therefore, at least one of $\alpha_1,\beta_1,\ldots,\alpha_n,\beta_n$ being nonzero is equivalent to $f$ not being almost everywhere equal to a constant function on $[0,T]$. This completes the proof of part~(ii) and hence the theorem.

\end{proof}

\section{Proof of Theorem~\ref{thm:asymptotic_dsq}}\label{appendix:thm5_2}
To prove Theorem~\ref{thm:asymptotic_dsq}, we proceed in three steps. First, we derive the DSQ gradient, as stated in Lemma~\ref{lemma:dsq_gradient}. Next, we establish the expectation and variance formulas for the DSQ and StableQAT schemes, respectively, which are given in Lemma~\ref{lem:exp-var}. Finally, we complete the proof of Theorem~\ref{thm:asymptotic_dsq} by combining these results.

\paragraph{DSQ details.} To address the non-differentiability of standard quantization, DSQ introduces a smooth asymptotic function $\phi(\cdot)$ that approximates each step of the uniform staircase quantizer. Specifically, the clipping range $[l, u]$ is uniformly partitioned into $2^b - 1$ intervals with width as $\Delta = \frac{u - l}{2^b - 1}$ and the $i$-th quantization interval defined as 
\begin{equation*}
P_i = [\, l + i\Delta,\; l + (i+1)\Delta), i \in \{0, \dots, 2^b - 2\}.
\end{equation*}
For an input $x \in P_i$, the smooth approximation is given by
\begin{equation}\label{eq:smooth-approx}
\phi(x) = s \cdot \tanh\big(k(x - m_i)\big),
\end{equation}
where $m_i=l + (i + \frac{1}{2})\Delta$ denotes the midpoint of $i$-th quantization interval $P_i$, and scaling factor
$s=\frac{1}{\tanh ((k\Delta)/2)}$ensures continuity of $\phi(\cdot)$ across adjacent intervals. The parameter $k$ controls the sharpness of the approximation, with larger values yielding a closer match to the hard staircase quantizer. To allow learnable control over the quantization sharpness, DSQ introduces a variable $\alpha \in (0,1)$, which measures the approximation gap between soft and hard quantization and is defined as
\begin{equation}\label{eq:alpha-parameterize}
\alpha = 1 - \tanh\left(\tfrac{1}{2}k\Delta\right).
\end{equation}
Under this parameterization, the asymptotic function is fully determined by $\alpha$ and $\Delta$ with
\begin{equation}\label{eq:relation-s-k-alpha}
s = \frac{1}{1 - \alpha},
\qquad
k = \frac{1}{\Delta}\ln\left(\frac{2 - \alpha}{\alpha}\right).
\end{equation}
Combining~\eqref{eq:alpha-parameterize} and~\eqref{eq:relation-s-k-alpha} allows us to have the identity
\begin{equation}\label{eq:alpha-identity}
\alpha = 1 - \tanh\left(\frac{1}{2}\ln\left(\frac{2 - \alpha}{\alpha}\right)\right).
\end{equation}
Based on the asymptotic function $\phi(\cdot)$, the differentiable soft quantization (DSQ) function is defined as
\begin{equation}\label{eq:dsq-def}
Q_{\mathrm{DSQ}}(x) =
\begin{cases}
l, & x < l, \\
u, & x > u, \\
l + \Delta\left(i + \dfrac{\phi(x) + 1}{2}\right), & x \in P_i.
\end{cases}
\end{equation}

Having established the DSQ formulation, we first derive the gradient of the DSQ function.
\begin{lemma}[DSQ Gradient]
\label{lemma:dsq_gradient}
Given the input $x \in P_i$, the gradient of the differentiable soft quantization function $Q_{\mathrm{DSQ}}(x)$ defined in~\eqref{eq:dsq-def} is
\begin{equation}\label{eq:dsq-grad}
\gdsq(x) = \frac{\partial Q_{\mathrm{DSQ}}(x)}{\partial x} = \frac{\ln\left(\frac{2-\alpha}{\alpha}\right)}{2(1-\alpha)} \cdot \mathrm{sech}^2\left(\frac{\ln\left(\frac{2-\alpha}{\alpha}\right)}{\Delta}(x - m_i)\right)
\end{equation}
\end{lemma}
\begin{proof}
Using the definition of DSQ function (See~\eqref{eq:dsq-def}) and the fact that $\frac{d}{du}\tanh(u) = \mathrm{sech}^2(u)$, we obtain that for all $x \in P_i$,
\begin{equation*}
\frac{\partial Q_{\mathrm{DSQ}}(x)}{\partial x} = \frac{\Delta}{2} \cdot s \cdot k \cdot \mathrm{sech}^2(k(x - m_i)).
\end{equation*}
Substituting $s = \frac{1}{1-\alpha}$ and $k = \frac{1}{\Delta}\ln\left(\frac{2 - \alpha}{\alpha}\right)$ provided in~\eqref{eq:relation-s-k-alpha} yields the result~\eqref{eq:dsq-grad}, which completes the proof.
\end{proof}


We first present several existing integral results, which will be used to derive the expectation and variance of the StableQAT and DSQ schemes, respectively.
\begin{lemma}\label{lem:useful-integrals}
The following integral identities hold:
\begin{enumerate}
\item[(i)] If $c^2 < 1$, then
\begin{equation*}
\int_{-\frac{\pi}{2}}^{\frac{\pi}{2}} \frac{1}{1+c\cos x} dx = \frac{4}{\sqrt{1-c^2}} \arctan\left(\sqrt{\frac{1-c}{1+c}}\right).
\end{equation*}
\item[(ii)] If $c^2 < 1$, then
\begin{equation*}
\int_{-\frac{\pi}{2}}^{\frac{\pi}{2}} \frac{1}{(1+c\cos x)^2} dx = \frac{4}{(1-c^2)^{3/2}}\arctan\left(\sqrt{\frac{1-c}{1+c}}\right) - \frac{2c}{1-c^2}.
\end{equation*}
\item[(iii)] Given any $c>0$, 
\begin{equation*}
\int_{-c}^{c} \text{sech}^4(x) dx = 2\left(\tanh(c) -\frac{\tanh^3(c)}{3}\right). 
\end{equation*}
\end{enumerate}
\end{lemma}
\begin{proof}
(i) Using the result from~\cite{gradshteyn2007table} and the fact that $c^2< 1$, we have that 
\begin{equation*}
\int_{-\pi/2}^{\pi/2} \frac{1}{1+c\cos x} dx = \left. \frac{2}{\sqrt{1 - c^2}} \arctan \left(\frac{(1-c)\tan\left(\frac{x}{2}\right)}{\sqrt{1-c^2}}\right) \right|_{-\pi/2}^{\pi/2} = \frac{4}{\sqrt{1-c^2}} \arctan\left(\sqrt{\frac{1-c}{1+c}}\right),
\end{equation*}
as claimed. \\
(ii) It follows from~\cite{gradshteyn2007table} and the fact that $c^2 < 1$ that
\begin{equation*}
\begin{aligned}
\int_{-\pi/2}^{\pi/2} \frac{1}{(1+c\cos x)^2} dx &= -\frac{1}{1-c^2} \left\{\frac{c\sin(x)}{1 + c\cdot\cos(x)} - \frac{2}{\sqrt{1 - c^2}} \arctan \left(\frac{(1-c)\tan\left(\frac{x}{2}\right)}{\sqrt{1-c^2}}\right)\right\}\Bigg|_{-\pi/2}^{\pi/2} \\
&= \frac{4}{(1-c^2)^{3/2}}\arctan\left(\sqrt{\frac{1-c}{1+c}}\right) - \frac{2c}{1-c^2},
\end{aligned}
\end{equation*}
as claimed. \\
(iii) Let $u = \tanh(x)$ and this gives $du = \mathrm{sech}^2(x)\ dx$. It then follows from the fact that $\mathrm{sech}^4(x) = \mathrm{sech}^2(x) \cdot (1 - \tanh^2(x))$ that
\begin{equation*}
\int_{-c}^{c} \mathrm{sech}^4(x)\, dx = \int_{\tanh(-c)}^{\tanh(c)} (1 - u^2)\, du = \left(u - \frac{u^3}{3}\right) \Bigg|_{\tanh(-c)}^{\tanh(c)} = 2\left(\tanh(c) -\frac{\tanh^3(c)}{3}\right),
\end{equation*}
as claimed. 
\end{proof}
With these integral identities, we can now derive closed-form expressions for the statistical properties of each gradient estimator. The following lemma provides the expectation and variance under uniform sampling.
\begin{lemma}[Expectation and variance]\label{lem:exp-var}
The expectation and variance of the gradient estimators under DSQ and StableQAT schemes are given as follows.
\begin{enumerate}
\item \textbf{(DSQ)}
\begin{equation*}
\begin{aligned}
\E_{\xi \sim U(l,u)}[g_{\mathrm{DSQ}}(x;\alpha)] = 1,
\Var_{\xi \sim U(l,u)}[g_{\mathrm{DSQ}}(x;\alpha)]= \frac{\ln\!\left(\frac{2-\alpha}{\alpha}\right)
\left(3-(1-\alpha)^2\right)}{6(1-\alpha)} - 1 .
\end{aligned}
\end{equation*}

\item \textbf{(StableQAT)}
Let $c = \sqrt{2}\pi A$ with $A \in \bigl(0,\frac{1}{\sqrt{2}\pi}\bigr)$. Then,
\begin{equation*}
\begin{aligned}
\E_{\xi \sim U(l,u)}[g_{\emph{\algacro{}}}(x;A)]
&= \frac{8}{\pi\sqrt{1-c^2}}
\arctan\!\left(\sqrt{\frac{1-c}{1+c}}\right) - 1, \\
\Var_{\xi \sim U(l,u)}[g_{\emph{\algacro{}}}(x;A)]
&= \frac{16}{\pi(1-c^2)^{3/2}}
\arctan\!\left(\sqrt{\frac{1-c}{1+c}}\right)
- \frac{8c}{\pi(1-c^2)} + 1
- \left(
\frac{8}{\pi\sqrt{1-c^2}}
\arctan\!\left(\sqrt{\frac{1-c}{1+c}}\right) - 1
\right)^2.
\end{aligned}
\end{equation*}
\end{enumerate}
\end{lemma}
\begin{proof}
For notation brevity, we use $\E[\cdot]$ to denote $\E_{\xi \sim U(l, u)}[\cdot]$ and $\Var[\cdot]$ to denote $\Var_{\xi \sim U(l, u)}[\cdot]$. Let $\beta = \In(\frac{2-\alpha}{\alpha})$.

\paragraph{DSQ Expectation.} 
We have from Lemma~\ref{lemma:dsq_gradient}, change of variable trick, the fact that range $[l,u]$ has $2^b-1$ interval with length $\Delta$, $\frac{d}{dx}\tanh(x) = \text{sech}^2(x)$, and~\eqref{eq:alpha-identity} that 
\begin{equation}\label{eq:dsq-first-moment}
\begin{aligned}
\E[\gdsq(\xi;\alpha)] &= \sum_{i=0}^{2^b - 2}\int_{m_i - \frac{\Delta}{2}}^{m_i + \frac{\Delta}{2}} \frac{\beta}{2(1-\alpha)} \cdot \mathrm{sech}^2\left(\frac{\beta}{\Delta}(\xi - m_i)\right) \cdot \frac{1}{u-l} d\xi \\
&= \sum_{i=0}^{2^b - 2}\int_{-\frac{\Delta}{2}}^{\frac{\Delta}{2}} \frac{\beta}{2(1-\alpha)} \cdot \mathrm{sech}^2\left(\frac{\beta}{\Delta}\xi \right) \cdot \frac{1}{u-l} d\xi \\
&= \cdot \frac{2^b - 1}{u-l}\int_{-\frac{\Delta}{2}}^{\frac{\Delta}{2}} \frac{\beta}{2(1-\alpha)} \cdot \mathrm{sech}^2\left(\frac{\beta}{\Delta}\xi \right) d\xi \\
&= \frac{1}{\Delta} \int_{-\frac{\Delta}{2}}^{\frac{\Delta}{2}} \frac{\beta}{2(1-\alpha)} \cdot \mathrm{sech}^2\left(\frac{\beta}{\Delta}\xi\right) d\xi \\
&= \frac{1}{\Delta} \cdot \frac{\beta}{2(1-\alpha)} \cdot \frac{\Delta}{\beta} \cdot \left.\tanh\left(\frac{\beta}{\Delta} \xi\right)\right|_{-\frac{\Delta}{2}}^{\frac{\Delta}{2}} \\
&=  \frac{1}{1-\alpha} \mathrm{tanh}\left(\frac{\beta}{2}\right) = \frac{1}{1-\alpha} \mathrm{tanh}\left(\frac{1}{2}\ln\left(\frac{2-\alpha}{\alpha}\right)\right)=1.
\end{aligned}
\end{equation}
\paragraph{DSQ Variance.} To compute DSQ variance, one first need to compute the second order moment of random variable $\gdsq(\xi;\alpha)$. We have from Lemma~\ref{lemma:dsq_gradient}, change of variable trick, the fact that range $[l,u]$ has $2^b-1$ interval with length $\Delta$, Lemma~\ref{lem:useful-integrals}(iii), and~\eqref{eq:alpha-identity} that 
\begin{equation}\label{eq:dsq-second-moment}
\begin{aligned}
\E[\gdsq^2(\xi;\alpha)] &= \sum_{i=0}^{2^b - 2}\int_{m_i - \frac{\Delta}{2}}^{m_i + \frac{\Delta}{2}} \left(\frac{\ln\left(\frac{2-\alpha}{\alpha}\right)}{2(1-\alpha)}\right)^2 \cdot \sech^4\left(\frac{\ln\left(\frac{2-\alpha}{\alpha}\right)}{\Delta}(\xi - m_i)\right) \cdot \frac{1}{u-l} d\xi \\
&= \sum_{i=0}^{2^b-2} \int_{-\frac{\Delta}{2}}^{\frac{\Delta}{2}} \left(\frac{\beta}{2(1-\alpha)}\right)^2 \sech^4\left(\frac{\beta}{\Delta} \xi\right) d\xi \\
&= \frac{1}{\Delta} \int_{-\frac{\beta}{2}}^{\frac{\beta}{2}} \left(\frac{\beta}{2(1-\alpha)}\right)^2 \cdot  \frac{\Delta}{\beta} \cdot \sech^4(\xi) d\xi \\
&= \frac{\beta}{4(1-\alpha)^2} \int_{-\frac{\beta}{2}}^{\frac{\beta}{2}} \sech^4(\xi) d\xi \\
&= \frac{\beta}{4(1-\alpha)^2}\left(2\tanh\left(\frac{\beta}{2}\right) - \frac{2}{3}\tanh^3\left(\frac{\beta}{2}\right)\right) \\
&= \frac{\beta}{4(1-\alpha)^2}\left(2(1-\alpha) - \frac{2}{3}(1-\alpha)^3\right) = \frac{\ln\left(\frac{2-\alpha}{\alpha}\right) \cdot (3 - (1-\alpha)^2)}{6(1-\alpha)}.
\end{aligned}
\end{equation}
It then follows from the definition of variance,~\eqref{eq:dsq-first-moment} and~\eqref{eq:dsq-second-moment} that
\begin{equation*}
\Var[\gdsq(\xi;\alpha)] = \E[\gdsq^2(\xi;\alpha)] - (\E[\gdsq(\xi;\alpha)])^2 = \frac{\ln\!\left(\frac{2-\alpha}{\alpha}\right)
\left(3-(1-\alpha)^2\right)}{6(1-\alpha)} - 1.
\end{equation*}
For $N = 1$, the StableQAT gradient is given as for all $\xi\in[l, u]$,
\begin{equation}\label{eq:StableQAT-grad-modified}
g_{{\algacro{}}}(\xi) = \frac{1 - \sqrt{2}\pi A\cos\left(\pi\left(\frac{\xi}{\Delta} + \left\lfloor\frac{\xi}{\Delta}\right\rceil\right)\right)}{1 + \sqrt{2}\pi A\cos\left(\pi\left(\frac{\xi}{\Delta} + \left\lfloor\frac{\xi}{\Delta}\right\rceil\right)\right)} = \frac{1 - c\cdot\cos\left(\pi\left(\frac{\xi}{\Delta} - \left\lfloor\frac{\xi}{\Delta}\right\rceil\right)\right)}{1 + c\cdot \cos\left(\pi\left(\frac{\xi}{\Delta} - \left\lfloor\frac{\xi}{\Delta}\right\rceil\right)\right)}.
\end{equation}
\paragraph{StableQAT Expectation.} Using~\eqref{eq:StableQAT-grad-modified}, the change of variable trick ($\theta = \frac{x}{\Delta} - \left\lfloor\frac{x}{\Delta}\right\rceil$), the fact that range $[l,u]$ has $2^b-1$ interval with length $\Delta$, and Lemma~\ref{lem:useful-integrals}(i) that
\begin{equation}\label{eq:StableQAT-first-moment}
\begin{aligned}
\E[g_{{\algacro{}}}(\xi;A)] &= \int_{l}^u g_{{\algacro{}}}(\xi) \cdot \frac{1}{u-l} d\xi \\
&= \sum_{i=0}^{2^b-3} \int_{m_i}^{m_{i+1}} g_{{\algacro{}}}(\xi) \cdot \frac{1}{u-l} d\xi + \int_{l}^{l+\frac{\Delta}{2}} g_{{\algacro{}}}(\xi) \cdot \frac{1}{u-l} d\xi + \int_{u-\frac{\Delta}{2}}^{u} g_{{\algacro{}}}(\xi) \cdot \frac{1}{u-l} d\xi \\
&= \sum_{i=0}^{2^b-3} \int_{-\frac{\pi}{2}}^{\frac{\pi}{2}} \frac{1- c\cdot \cos(\theta)}{1+ c\cdot \cos(\theta)} \cdot \frac{\Delta}{\pi(u-l)} d\theta + \int_{0}^{\frac{\pi}{2}} \frac{1- c\cdot \cos(\theta)}{1+ c\cdot \cos(\theta)} \cdot \frac{\Delta}{\pi(u-l)} d\theta + \int_{-\frac{\pi}{2}}^{0} \frac{1- c\cdot \cos(\theta)}{1+ c\cdot \cos(\theta)} \cdot \frac{\Delta}{\pi(u-l)} d\theta \\
&= \frac{(2^b - 1)\Delta}{u-l} \cdot \frac{1}{\pi} \int_{-\frac{\pi}{2}}^{\frac{\pi}{2}} \frac{1- c\cdot \cos(\theta)}{1+ c\cdot \cos(\theta)} d\theta \\
&= \frac{1}{\pi} \int_{-\frac{\pi}{2}}^{\frac{\pi}{2}} \frac{1- c\cdot \cos(\theta)}{1+ c\cdot \cos(\theta)} d\theta \\
&= \frac{1}{\pi} \int_{-\frac{\pi}{2}}^{\frac{\pi}{2}} \left(\frac{2}{1 + c\cos\theta} - 1\right) d\theta \\
&= \frac{8}{\pi\sqrt{1-c^2}} \arctan\left(\sqrt{\frac{1-c}{1+c}}\right) - 1. 
\end{aligned}
\end{equation}
\paragraph{StableQAT Variance.} Using the change of variable trick ($\theta = \frac{\xi}{\Delta} - \left\lfloor\frac{\xi}{\Delta}\right\rceil$), the fact that range $[l,u]$ has $2^b-1$ interval with length $\Delta$, and Lemma~\ref{lem:useful-integrals}(i)-(ii) that
\begin{equation}\label{eq:StableQAT-second-moment}
\begin{aligned}
 & \ \E[g_{{\algacro{}}}^2(\xi;A)] \\
=& \ \int_{l}^u g_{{\algacro{}}}^2(\xi) \cdot \frac{1}{u-l} d\xi \\
=& \ \sum_{i=0}^{2^b-3} \int_{m_i}^{m_{i+1}} g_{{\algacro{}}}^2(\xi) \cdot \frac{1}{u-l} d\xi + \int_{l}^{l+\frac{\Delta}{2}} g_{{\algacro{}}}^2(\xi) \cdot \frac{1}{u-l} d\xi + \int_{u-\frac{\Delta}{2}}^{u} g_{{\algacro{}}}^2(\xi) \cdot \frac{1}{u-l} d\xi \\
=& \ \sum_{i=0}^{2^b-3} \int_{-\frac{\pi}{2}}^{\frac{\pi}{2}} \left(\frac{1- c\cdot \cos(\theta)}{1+ c\cdot \cos(\theta)}\right)^2 \cdot \frac{\Delta}{\pi(u-l)} d\theta + \int_{0}^{\frac{\pi}{2}} \left(\frac{1- c\cdot \cos(\theta)}{1+ c\cdot \cos(\theta)}\right)^2 \cdot \frac{\Delta}{\pi(u-l)} d\theta + \int_{-\frac{\pi}{2}}^{0} \left(\frac{1- c\cdot \cos(\theta)}{1+ c\cdot \cos(\theta)}\right)^2 \cdot \frac{\Delta}{\pi(u-l)} d\theta \\
=& \ \frac{(2^b - 1)\Delta}{\pi(u-l)} \int_{-\frac{\pi}{2}}^{\frac{\pi}{2}} \left(\frac{1- c\cdot \cos(\theta)}{1+ c\cdot \cos(\theta)}\right)^2 d\theta \\
=& \ \frac{1}{\pi} \int_{-\frac{\pi}{2}}^{\frac{\pi}{2}} \left(\frac{1- c\cdot \cos(\theta)}{1+ c\cdot \cos(\theta)}\right)^2 d\theta \\
=& \ \frac{1}{\pi} \int_{-\frac{\pi}{2}}^{\frac{\pi}{2}} \frac{4}{(1 + c \cdot \cos\theta)^2} - \frac{4}{1 + c \cdot \cos\theta} + 1 \ d\theta \\
=& \ \frac{16}{\pi(1-c^2)^{3/2}}
\arctan\!\left(\sqrt{\frac{1-c}{1+c}}\right)
- \frac{8c}{\pi(1-c^2)} + 1 - \frac{16}{\pi\sqrt{1-c^2}}
\arctan\!\left(\sqrt{\frac{1-c}{1+c}}\right).
\end{aligned}
\end{equation}
It then follows from~\eqref{eq:StableQAT-first-moment} and~\eqref{eq:StableQAT-second-moment} that 
\begin{equation*}
\begin{aligned}
\Var[g_{{\algacro{}}}(\xi;A)] &= \E[g_{{\algacro{}}}^2(\xi;A)] - \left(\E[g_{{\algacro{}}}(\xi;A)]\right)^2 \\
&= \frac{16c^2}{\pi(1-c^2)^{3/2}}
\arctan\!\left(\sqrt{\frac{1-c}{1+c}}\right)
- \frac{8c}{\pi(1-c^2)} + 1
- \left(
\frac{8}{\pi\sqrt{1-c^2}}
\arctan\!\left(\sqrt{\frac{1-c}{1+c}}\right) - 1
\right)^2. 
\end{aligned}
\end{equation*}
\end{proof}

Now, we are ready to prove Theorem~\ref{thm:asymptotic_dsq}.
\begin{proof}[Proof of Theorem~\ref{thm:asymptotic_dsq}]
We establish each limit separately, drawing upon the closed-form expressions derived in Lemma~\ref{lem:exp-var}. Throughout, we denote $c = \sqrt{2}\pi A$ and note that as $A \to \left(\frac{1}{\sqrt{2}\pi}\right)^-$, we have $c \to 1^{-}$.

\textbf{DSQ expectation limit.}
From Lemma~\ref{lem:exp-var}, the expected value of the DSQ gradient satisfies
$\E_{\xi \sim U(l,u)}[\gdsq(\xi)] = 1$ for all $\alpha \in (0,1)$. Since this identity is independent of $\alpha$, it follows that
\begin{equation*}
\lim_{\alpha \to 0^+} \E[\gdsq(\xi;\alpha)] = 1.
\end{equation*}

\paragraph{StableQAT expectation limit.} From Lemma~\ref{lem:exp-var}, the expectation of the StableQAT gradient is given by
\begin{equation*}
\E_{\xi \sim U(l,u)}[g_{{\algacro{}}}(\xi;A)] = \frac{8}{\pi\sqrt{1-c^2}} \arctan\left(\sqrt{\frac{1-c}{1+c}}\right) - 1.
\end{equation*}
It then follows from the fact that $\lim_{x \to 0} \frac{\arctan(x)}{x} = 1$ that
\begin{equation}\label{eq:T2c-limit}
\begin{aligned}
\lim_{A \to \left(\frac{1}{\sqrt{2}\pi}\right)^-} \E[g_{{\algacro{}}}(x)] &= \lim_{c \to 1^-} \frac{8}{\pi\sqrt{1-c^2}} \arctan\left(\sqrt{\frac{1-c}{1+c}}\right) - 1 \\
&= \lim_{c \to 1^{-}} \frac{\arctan\left(\sqrt{\frac{1-c}{1+c}}\right)}{\left(\sqrt{\frac{1-c}{1+c}}\right)} \cdot \frac{8}{\pi\sqrt{1-c^2}} - 1 = \frac{4}{\pi} - 1.
\end{aligned}
\end{equation}

\paragraph{DSQ variance limit.}
From Lemma~\ref{lem:exp-var}, the variance of the DSQ gradient is given by
\begin{equation}\label{eq:dsq-variance}
\Var_{\xi \sim U(l,u)}[\gdsq(\xi;\alpha)] = \frac{\ln\left(\frac{2-\alpha}{\alpha}\right)\bigl(3 - (1-\alpha)^2\bigr)}{6(1-\alpha)} - 1.
\end{equation}
We first analyze the asymptotic behavior of each factor as $\alpha \to 0^{+}$ as follows.
\begin{equation}\label{eq:useful-limits}
\begin{aligned}
\lim_{\alpha \to 0^+}\ln\left(\frac{2-\alpha}{\alpha}\right) = +\infty, \lim_{\alpha \to 0^+} (1-\alpha) = 1, \lim_{\alpha \to 0^+} (3 - (1-\alpha)^2) = 2.
\end{aligned}
\end{equation}
Combining~\eqref{eq:dsq-variance} and~\eqref{eq:useful-limits} allows us to have that
\begin{equation*}
\lim_{\alpha \to 0^{+}} \Var[\gdsq(\xi;\alpha)] = \lim_{\alpha \to 0^{+}} \frac{\ln\left(\frac{2-\alpha}{\alpha}\right)\bigl(3 - (1-\alpha)^2\bigr)}{6(1-\alpha)} - 1= +\infty.
\end{equation*}
\paragraph{StableQAT variance limit.} From Lemma~\ref{lem:exp-var}, the variance of the StableQAT gradient is given by 
\begin{equation*}
\Var_{\xi \sim U(l,u)}[g_{{\algacro{}}}(\xi;A)] = \underbrace{\frac{16c^2}{\pi(1-c^2)^{3/2}}
\arctan\!\left(\sqrt{\frac{1-c}{1+c}}\right)
- \frac{8c}{\pi(1-c^2)}}_{T_1(c)} + 1
- \underbrace{\left(
\frac{8}{\pi\sqrt{1-c^2}}
\arctan\!\left(\sqrt{\frac{1-c}{1+c}}\right) - 1
\right)^2}_{T_2(c)}. 
\end{equation*}
We note that $\lim_{c \to 1^-} T_2(c)$ has been derived in~\eqref{eq:T2c-limit}. To this end, it remains to compute $\lim_{c \to 1^-} T_1(c)$. Let $u^2 = \frac{1-c}{1+c}$ and this implies that
\begin{equation}\label{eq:c-u-relation}
c = \frac{1-u^2}{1+u^2} \text{ and } 1-c^2 = \frac{4u^2}{(1+u^2)^2}.
\end{equation}
Then, using~\eqref{eq:c-u-relation} and Taylor series expansion of function $\arctan(x)$ at $x= 0$, we have that
\begin{equation}\label{eq:T1c-limit}
\begin{aligned}
\lim_{c \to 1^-} T_1(c) =\;& \lim_{u \to 0^+} \frac{16(1-u^2)^2}{\pi\,8u^3}\arctan(u)
-
\frac{8\left(\frac{1-u^2}{1+u^2}\right)}{\pi\left(\tfrac{4u^2}{(1+u^2)^2}\right)} \\
=\;&
\lim_{u \to 0^+}
\frac{16(1-u^2)^2(1+u^2)}{\pi\,8u^3}
\left(u - \frac{u^3}{3} + O(u^5)\right)
-
\frac{8(1-u^2)(1+u^2)}{4\pi u^2} \\
=\;&
\lim_{u \to 0^+}
\frac{16(1-u^2)^2(1+u^2)u}{24\pi}
-
\frac{8(1-u^2)(1+u^2)}{4\pi u^2} \\
=\;&
\lim_{u \to 0^+}
\frac{2(1-u^2-u^4+u^6)}{3\pi} - \frac{2(1-2u^2+u^4)}{\pi u^2} \\
=\;&
\lim_{u \to 0^+}
\left(-\frac{8}{3\pi} + O(u^2)\right) = -\frac{8}{3\pi}.
\end{aligned}
\end{equation}
Combining~\eqref{eq:T2c-limit} and~\eqref{eq:T1c-limit} allows us to have that
\begin{equation*}
\lim_{A \to \left(\frac{1}{\sqrt{2}\pi}\right)^-}
\Var_{\xi \sim U(l,u)}\!\left[g_{{\algacro{}}}(\xi;A)\right] = \lim_{c \to 1^-} T_1(c) + 1 - T_2(c) = \frac{16}{3\pi} - \frac{16}{\pi^2}.
\end{equation*}
\end{proof}

\end{document}